\documentclass{article}

\usepackage{amsmath,amssymb,amsfonts,amsthm}
\usepackage[numbers]{natbib}
\usepackage[colorlinks,citecolor=blue,urlcolor=blue]{hyperref}
\usepackage[utf8]{inputenc}
\usepackage[T1]{fontenc}
\usepackage{microtype}
\usepackage{fullpage}
\usepackage{mathrsfs}
\usepackage{dsfont}
\usepackage{todonotes}
\usepackage{subcaption}

\newtheorem{definition}{Definition}[section]
\newtheorem{lemma}{Lemma}[section]
\newtheorem{theorem}{Theorem}[section]
\newtheorem{proposition}{Proposition}[section]

\newtheorem{remark}{Remark}[section]

\newcommand{\eqdef }{\ensuremath{\stackrel{\mbox{\upshape\tiny def.}}{=}}}
\newcommand{\ReLU}{\operatorname{ReLU}}
\newcommand{\ICG}{\mathrm{IC_G}}
\newcommand{\ICR}{\mathrm{IC_R}}
\newcommand{\AG}{\mathrm{A_G}}
\newcommand{\AR}{\mathrm{A_R}}

\usepackage{xcolor}
\definecolor{almostblack}{RGB}{20,20,20}
\definecolor{charcoalblack}{RGB}{28,28,28}

\usepackage[most]{tcolorbox}

\newtcolorbox{questionbox}{
  enhanced,
  hbox,
  colback=gray!6,
  colframe=gray!50,
  boxrule=0.4pt,
  arc=1.5mm,
  left=3mm,
  right=3mm,
  top=1.2mm,
  bottom=1.2mm,
  before=\begin{center},
  after=\end{center}
}
\usepackage{tikz}
\usetikzlibrary{arrows.meta}
\usepackage{xcolor,amsmath}

\definecolor{nearblack}{RGB}{20,20,20}
\definecolor{lightgraytext}{RGB}{175,175,175}
\definecolor{gridgray}{RGB}{190,190,190}

\definecolor{myblue}{RGB}{35,85,205}
\definecolor{myred}{RGB}{215,60,50}

\definecolor{hatone}{RGB}{170,190,70}
\definecolor{hattwo}{RGB}{215,70,55}
\definecolor{hatthree}{RGB}{60,160,60}
\definecolor{hatfour}{RGB}{235,155,45}
\definecolor{hatfive}{RGB}{200,55,50}
\definecolor{hatsix}{RGB}{220,120,50}

\definecolor{levelone}{RGB}{85,150,60}
\definecolor{leveltwo}{RGB}{245,150,30}
\definecolor{levelthree}{RGB}{235,95,35}
\definecolor{levelfour}{RGB}{210,35,35}

\title{Adaptivity Under Realizability Constraints:\\
Comparing In-Context and Agentic Learning}

\author{Anastasis Kratsios\thanks{McMaster University and the Vector Institute, Main St., Hamilton, L8S 4L8, Ontario, Canada, email: \texttt{kratsioa@mcmaster.ca}} \footnotemark[4] \and A. Martina Neuman \thanks{University of Vienna, Faculty of Mathematics, Kolingasse 14-16,
    1090 Wien, Austria, 
    e-mail: \texttt{anh.martina.neuman@univie.ac.at}
  } \footnotemark[4] \and Philipp Petersen \thanks{
    University of Vienna,
    Faculty of Mathematics and Research Network Data Science @ Uni Vienna, Kolingasse 14-16,
    1090 Vienna, Austria, 
    e-mail: \texttt{philipp.petersen@univie.ac.at}
  } \thanks{All authors contributed equally}}
\date{May 2026}

\begin{document}

\maketitle

\begin{abstract}
We compare in-context learning with fixed queries and 
agentic learning with adaptive queries for uniform
approximation of task families. 
We consider two settings: 
an unrestricted regime,  
where querying and approximation are  
arbitrary functions, and a realizable regime, 
where we require these operations to be implemented by ReLU neural networks. 
In both settings, adaptivity never hinders approximation performance. 
However, this advantage can change when one passes from the unrestricted regime to the realizable regime.
We identify four distinct approximation scenarios, each witnessed by an explicit task family: (a) no advantage of adaptivity; (b) an advantage in the unrestricted regime that persists under ReLU realizability; (c) an advantage that arises only under realizability; and (d) an advantage that disappears under realizability.
This demonstrates that representational constraints interact profoundly with the effect of adaptivity.
\end{abstract}

\section{Introduction}

Modern learning systems exhibit 
well-documented few-shot~\cite{brown2020,li2025transformers} and \emph{multi-task}~\cite{Caruana1997Multitask,ArgyriouEvgeniouPontil2008ConvexMTL,MishraEtAl2022NaturalInstructions,WangEtAl2022SuperNaturalInstructions} learning capabilities. 
These recent advances are largely attributed to their ability to access task information through prompts \cite{brown2020}, retrieved examples \cite{lewis2020rag}, memory \cite{park2023generative}, or sequential interaction with an environment \cite{yao2023react}. 
Fixed prompts are an instance of \emph{in-context learning} (ICL), while sequential interaction leads to \emph{agentic learning}. 
For a mathematically tractable treatment, we adopt a viewpoint in which sequential interaction consists of adaptive query selection based on past observations, whereas prompts are fixed queries specified a priori.
The central question of this article is simple: \emph{when does sequential interaction actually help, and how does that answer change once realizability constraints are imposed?}

To this end, we compare four regimes, defined formally in Section~\ref{sec:setup}:
\begin{align*}
&\ICG \;\text{(general in-context)},&
\qquad
&\ICR \;\text{(ReLU in-context)},\\
&\AG \,\,\;\text{(general agentic)},&
\qquad
&\AR \,\,\;\text{(ReLU agentic)}.
\end{align*}
For our purposes, an in-context learner is a function that can access a fixed (but arbitrary) collection of samples, whereas an agentic learner (or ``agent'') selects its samples adaptively based on past observations.
Both learners 
are required to approximate an entire given class of functions. 
Along a second axis, we distinguish between 
\emph{general} approximators and \emph{ReLU-realizable} (or simply \emph{realizable}) approximators. 
A general approximator is subject to no representational restriction; a realizable one requires both the querying mechanism and the final predictor to be implemented by bounded-size ReLU networks.
Formal definitions are provided in Section \ref{sec:setup}. 

To understand the relationship in terms of approximation power between the introduced regimes, 
we consider uniform approximation over a given \emph{full function class} and assume equal sample budgets and, where applicable, equal computational budgets for the underlying ReLU networks.
An immediate observation is the monotonicity
\begin{align}
\label{eq:adaptivityCannotHurtAnyone}
\ICG\leq \AG,
\quad\text{ and }\quad
\ICR\leq \AR,
\end{align}
which holds for \emph{all} function classes. 
Here ``$\leq$'' means that the method on the right does not incur a larger worst-case uniform approximation error than the one on the left. 
The inequality \eqref{eq:adaptivityCannotHurtAnyone} follows from the fact that adaptive sampling can always be ignored, so an agent can imitate any fixed sampling rule.
We formalize this fundamental fact in Proposition~\ref{prop:monotonicity}.
In the sequel, we will also use the symbol ``$=$'' to denote that both methods yield the same worst-case error (over an underlying function class) and $<$ to denote that the method on the right yields a strictly smaller worst-case error than the one on the left. 
We give a formal definition of these relations in Definition \ref{def:formalDefRelations}.

Apart from the case $\ICG=\AG$ and $\ICR=\AR$---which already occurs for every singleton task
family---
in which adaptivity yields no improvement, only three nontrivial configurations remain:

\begin{center}
\begin{tabular}{lcc}
\hline
\textbf{Phenomenon} & \textbf{General} & \textbf{ReLU-realizable} \\
\hline
adaptivity advantage survives realizability & $\ICG<\AG$ & $\ICR<\AR$ \\
realizability generates adaptivity advantage & $\ICG=\AG$ & $\ICR<\AR$ \\
adaptivity advantage removed by realizability & $\ICG<\AG$ & $\ICR=\AR$ \\
\hline
\end{tabular}
\end{center}

In the main results, we isolate all three of these nontrivial configurations.
Concretely,
Theorems~\ref{thm:path}, \ref{thm:value}, and \ref{thm:address} identify function classes realizing each of these patterns, which may be distinguished as follows: 
\begin{itemize}
    \item \emph{adaptivity reveals more information}: on \emph{cubical-path} tasks, adaptive querying
    reveals strictly more task information than every fixed context, and this advantage survives the
    ReLU-realizability restriction; see Section~\ref{s:Main__ss:pathada_realizability};
    \item \emph{adaptivity bypasses internal computation}: on the \emph{pointed-value} family,
    unrestricted in-context and unrestricted agentic learning are equally powerful, but a ReLU-realizable
    agentic learner can bypass a hard internal computation that a ReLU-realizable in-context learner needs to perform; 
    see Section~\ref{s:Main__ss:realizabilityonly};
    \item 
    \emph{adaptivity is limited by hard computation}:
    on the \emph{address-spike} family,
    unrestricted adaptivity provides strict approximation gains, 
    but now the informative adaptive query location is computationally hard to determine; 
    consequently, 
    bounded-size 
    ReLU agents 
    lose these gains; 
    see
    Section~\ref{s:Main__ss:killed-by-realizability}.
\end{itemize}

This yields the promised conceptual classification: \emph{once realizability is taken into account,
adaptivity advantages can survive, appear, or disappear}.

\section{Related work}

\paragraph{In-context learning.}
Our notion of in-context learning is motivated by the modern few-shot prompting paradigm for large
language models, where task information is supplied through examples or instructions at inference time rather than through parameter updates \cite{brown2020}. 
On the theoretical side, recent work studies what 
function classes can learn from such fixed in-context examples for \emph{structured families}, typically exploiting shared low complexity across tasks, particularly in transformer architectures \cite{akyurek2023what, garg2022}. 
While related in spirit, we use structured families from a different perspective: to highlight the role of adaptive information acquisition and to characterize how realizability alters its advantage.
In our formulation, this corresponds to an \emph{agent}, possibly realized by a ReLU network, that adaptively selects which inputs to query and updates its predictor based on the observed responses.

We distinguish our approximation-theoretic lens from the statistical lens that typically dominates the ICL literature. In particular, our focus differs from work on scaling laws~\cite{bordelon2025theory}, simplified attention mechanisms that learn linear tasks in context~\cite{wu2023many,zhang2024trained}, the low-frequency inductive bias of trained transformers~\cite{yang2025provable}, and algorithmic perspectives in which ICL implements an implicit learning algorithm within the forward pass~\cite{vonoswald2023what} or, more broadly, a meta-learning procedure that maps in-context examples directly to a predictor or its parameters~\cite{chen2022transformers}.
These approaches operate on a given prompt of in-context examples and therefore do not address adaptive information acquisition or how it interacts with architectural constraints and realizability.

\paragraph{Adaptive sampling, active learning, and information-based complexity.}
The adaptive 
component of our learning model is closely related to classical \emph{query learning} and \emph{active learning}, where a learner chooses which information to reveal 
based on past observations, viewed as sequentially acquired examples.
Membership- and equivalence-query models already highlight the power of adaptive access to an unknown target through sharp bounds on \emph{query complexity}, see e.g. \cite{angluin1987}, while the active learning literature studies how 
sequential sample selection reduces label complexity for generalization \cite{cohn1994,dasgupta2011}. 
Adaptive sampling frameworks similarly study how measurements can be chosen on-the-fly to improve estimation accuracy for structured signal classes \cite{castro2009active, sung1994active}.
More broadly, 
our setting is most closely aligned with the framework of \emph{information-based complexity}, which studies the
intrinsic difficulty of approximation 
under partial information and restricted
access to the target object. 
This framework was formalized in the 1980s in \cite{packel1987ibc, traub1988ibc} and further developed in \cite{novak2008tractability}, where approximation under partial information is modeled via prescribed \emph{information operators} (typically a collection of linear or nonlinear functionals), and complexity is measured by the minimal achievable error.
Modern developments consider a range of information models, including limited bit encodings \cite{kolmogorov1959varepsilon}, linear measurements \cite{pinkus2012n, krieg2021function, adcock2025optimal}, and nonlinear (e.g., $1$-Lipschitz) queries \cite{cohen2022optimal}.

\paragraph{Multi-tasking and connections to operator learning.}
The shift from the classical universal approximation perspective, where only a single task is approximated, to the approximation of task families remains only partially explored in ICL. 
This places our work within the broader context of approximation theory for \emph{operator learning}, where many results admit an interpretation as ``multi-task'' approximation guarantees~\cite{furuya2025one, KovachkiEtAl2023NeuralOperator, KovachkiLanthalerMishra2021FNOTheory, LanthalerMishraKarniadakis2022DeepONetError, kratsios2025generative, LiEtAl2021FNO}.
Those results, together with the first ICL approximation result of~\cite{li2025transformers}, suggest that multi-task learning may be subject to significant 
information-theoretic bottlenecks, 
similar to those encountered in neural operator approaches to simultaneously solving families of scientific computing problems~\cite{lanthaler2026parametric}. 
The objective of our paper, however, is not to characterize worst-case scaling laws of arbitrary multi-task approximation frameworks, but to isolate how adaptive information acquisition interacts with realizability constraints.
This perspective extends classical results on non-adaptive approximation by MLPs~\cite{Barron1993UniversalApproximation,mhaskar2017and}, suggesting that beyond function class complexity, the structure of information acquisition---such as adaptivity---also plays a central role.

\paragraph{Approximation-theoretic setting.}
The contribution of this note is to compare two axes---adaptivity and ReLU realizability---within a uniform approximation framework. This allows us to study how adaptivity interacts with representational constraints, which is central to the phenomena studied here. Imposing realizability constraints, in turn, places the discussion in the well-studied setting of ReLU networks, where both upper and lower bounds are available; see~\cite{petersen2024mathematical} for an overview.

We note that many of our results can be extended to multi-head transformers, and are therefore not MLP-specific, via our MLP-to-transformer conversion; see Proposition~\ref{prop:transformerification__SparseVersion} and the discussion in Remark \ref{rem:transformers}. 
This mirrors~\cite[Proposition 11]{kratsios2025context} for our standard notion of attention, and is consistent with similar conversion results for CNNs~\cite{petersen2020equivalence}, spiking neural networks~\cite{dold2026equivalence, neuman2025stable}, and transformers~\cite{hu2026transformerapproximationsrelus}. 

\section{Setup}\label{sec:setup}

Throughout, $\mathcal X\subseteq \mathbb R^d$ is compact, $\mathcal{T}\subseteq\{f:\mathcal X\to
\mathbb R\}$ is a family of tasks, and approximation is measured in the uniform norm
$\|\cdot\|_{L^\infty(\mathcal X)}$.  We consider $\operatorname{ReLU}$-MLPs and transformers with multi-head attention and $\operatorname{ReLU}$ activation functions; see Appendix~\ref{s:models} for standard definitions. 

Let $f\in\mathcal T$. For $n\in\mathbb N$ and query points $x_1,\dots,x_n\in\mathcal X$, we write
\[
C_n(f;x_1,\dots,x_n)\eqdef  ((x_i,f(x_i)))_{i=1}^n
\]
for the ordered \emph{context} gathered from the task $f$. 
Since we study the limits of 
adaptivity under realizability constraints, we do not include noise in the measurements $f(x_i)$, so as to isolate the core phenomena.
This is in part in line with standard approximation theory~\cite{yarotsky2017,petersen2018optimal,kratsios2022universal}.
Uniform (worst-case) approximation results for neural networks are typically formulated in this noise-free setting, while quantitative results incorporating noise remain comparatively limited; see, e.g., \cite{kratsios2025beyond}.

\begin{definition}[In-context learners]
Fix a sample budget $N\in\mathbb N$. A \emph{general in-context learner} $\Psi_{\rm IC}$ consists of fixed query points $x_1,\dots,x_N\in\mathcal X$ and a predictor
\[
 \widehat{F}_{\Psi_{\rm IC}}:(\mathcal X\times\mathbb R)^N\times\mathcal X\to\mathbb R,
\]
which, for a task $f\in\mathcal T$, outputs the function
\[
x\mapsto \Psi_{\rm IC}(f)(x) \eqdef 
\widehat{F}_{\Psi_{\rm IC}}
(C_N(f;x_1,\dots,x_N),x).
\]
We say that $\Psi_{\rm IC}$ is a \emph{ReLU-realizable in-context learner} (or \emph{realizable}) if it is represented by a bounded-size ReLU network. 
The size (i.e.\ number of non-zero parameters or network weights) of the ReLU network involved is defined as the size of $\Psi_{\rm IC}$.
\end{definition}

\begin{definition}[Agentic learners]
Fix a query budget $N\in\mathbb N$. A \emph{general agentic learner} $\Psi_{\rm A}$ consists of an initial
query point $x_1\in\mathcal X$, adaptive query maps
\begin{equation} \label{querymaps}
    q_n:(\mathcal X\times\mathbb R)^{n-1}\to\mathcal X,
    \qquad n=2,\dots,N,
\end{equation}
and a final predictor
\[
\widehat{F}_{\Psi_{\rm A}}:(\mathcal X\times\mathbb R)^N\times\mathcal X\to\mathbb R.
\]
For a task $f\in\mathcal T$, the resulting query sequence is defined recursively by
\[
x_1(f)\eqdef  x_1,
\qquad
x_n(f)\eqdef  q_n\bigl(C_{n-1}(f;x_1(f),\dots,x_{n-1}(f))\bigr),
\quad n=2,\dots,N,
\]
and the learner $\Psi_{\rm A}$ outputs the function
\begin{equation} \label{finpredictor}
    x\mapsto \Psi_{\rm A}(f)(x) \eqdef 
 \widehat{F}_{\Psi_{\rm A}}
 \bigl(C_N(f;x_1(f),\dots,x_N(f)),x\bigr).
\end{equation}
We say that $\Psi_{\rm A}$ is a \emph{ReLU-realizable agentic learner} (or \emph{realizable}) if the adaptive query maps and final predictor can all be
implemented by bounded-size ReLU networks. 
The size (i.e.\ number of non-zero parameters or network weights) of the largest of the involved ReLU networks is called the size of $\Psi_{\rm A}$.
\end{definition}

\begin{remark}
We do \emph{not} distinguish notationally between general learners, in-context or agentic, and their realizable counterparts.
In particular, the symbols $\Psi_{\rm IC}$ and $\Psi_{\rm A}$ may denote either type.
\end{remark}

The complexity of our networks is quantified via resource budgets. We distinguish between a \emph{query budget} $N$, corresponding to the number of samples, and a \emph{weight budget} $m$, corresponding to the network size (for realization). 
In the realizable regimes $(\ICR, \AR)$, both budgets, i.e. $(N,m)$, are imposed. In contrast, for the general regimes $(\ICG, \AG)$ no restriction on the class of functions used is imposed, so a weight budget is not applicable. There, only the query budget $N$ is constrained. All comparisons are made under matching query budgets, and, when applicable, matching weight budgets.
We now define a comparison relation between the regimes $\ICG, \AG, \ICR$, and $\AR$. 
%

\begin{definition} \label{def:formalDefRelations}
Fix a task family $\mathcal{T}$, a query budget $N\in\mathbb N$, and a weight budget $m\in\mathbb N$.
For $\{\mathsf R_1,\mathsf R_2\}\in \{\{\ICG,\AG\}, \{\ICR, \AR\}, \{\ICG,\ICR\}, \{\AG, \AR\}\}$, we write
\begin{equation*} 
    \mathsf R_1\leq \mathsf R_2
\end{equation*}
if for every learner $\Psi_1$ 
in regime $\mathsf R_1$ with the prescribed budgets there exists a learner $\Psi_2$ 
in regime $\mathsf R_2$ with the same prescribed budgets such that
\[
\sup_{f\in\mathcal T}\|\Psi_2(f)-f\|_{L^{\infty}(\mathcal X)}
\leq
\sup_{f\in\mathcal T}\|\Psi_1(f)-f\|_{L^{\infty}(\mathcal X)}.
\]
Here, the budgets are understood as follows: 
\begin{itemize}
    \item in the general regimes $\{\ICG, \AG\}$ or in the cross-type regimes $\{\ICG,\ICR\}, \{\AG, \AR\}$, the query budget $N$ is imposed;
    \item in the realizable regimes $\{\ICR, \AR\}$, both query and weight budgets $(N,m)$ are imposed.
\end{itemize}
We write $\mathsf R_1=\mathsf R_2$ if both $\mathsf R_1\leq \mathsf R_2$ and $\mathsf R_2\leq \mathsf R_1$ hold, and we write $\mathsf R_1<\mathsf R_2$ if $\mathsf R_1\leq \mathsf R_2$ holds but $\mathsf R_2\leq \mathsf R_1$ does not.
\end{definition}

We emphasize that all relations are defined relative to a fixed task family and the relevant resource budgets; in particular, no comparison is meaningful without specifying them.

\begin{remark}
Technically the relation $\mathsf R_1\leq \mathsf R_2$ from Definition~\ref{def:formalDefRelations} can be extended to the case $\{\mathsf R_1,\mathsf R_2\}\in \{\{\ICG,\AR\}, \{\ICR, \AG\}\}$ under appropriate budget imposition. We do not directly consider these relations here, as they confound the effects of adaptivity and realizability and therefore do not isolate the contribution of adaptivity.
\end{remark}

\begin{remark} \label{rem:trivialRelation}
We note for future reference the obvious relations, 
\begin{align*}
    \ICR \leq \ICG \quad \text{ and } \quad \AR \leq \AG,
\end{align*}
which follow from Definition~\ref{def:formalDefRelations}, since every realizable learner is a special case of a general learner.
\end{remark}

Our arguments often require a hard-to-approximate function. The following remark establishes the existence of one. 

\begin{remark}\label{rem:unattainableFunctions}
We shall repeatedly use the following standard fact: for every compact domain $\mathcal X\subseteq\mathbb R^d$, every weight budget $m$, and every compact interval $[a,b]$, there exists a continuous function with values in $[a,b]$ such that no ReLU network with at most $m$ weights can approximate it uniformly with error strictly smaller than $(b-a)/2$. To see this, assume toward a contradiction that every continuous function $f:\mathcal X\to[a,b]$ admits a ReLU network $\Phi_f$ with worst-case error smaller than $(b-a)/2$. 
Then for every $K \in \mathbb N$, every set of distinct points $(x_i)_{i=1}^K\subseteq \mathcal X$, and every label vector $(y_i)_{i=1}^K \in \{0,1\}^K$, there exists a continuous function $f:\mathcal X\to[0,1]$ such that $f(x_i)=y_i$ for all $i$. By assumption, the network $\Phi_{a+(b-a)f}$ then satisfies $\Phi_{a+(b-a)f}(x_i) > a+(b-a)/2$ if $y_i = 1$ and $\Phi_{a+(b-a)f}(x_i) < a+(b-a)/2$ if $y_i = 0$. Hence the class of ReLU networks with at most $m$ weights, composed with thresholding at $a+(b-a)/2$, shatters every finite set of distinct points in $\mathcal X$. This implies infinite VC dimension, contradicting \cite[Theorem 8.4]{AnthonyBartlett1999}. 
\end{remark}

In the sequel, we also occasionally use the following result: the existence of an $\varepsilon$-uniform $\operatorname{ReLU}$-MLP approximator $\operatorname{Mult}_{\varepsilon}$ of the binary multiplication function on $[0,1]^2$ 
with $\mathcal{O}(\log(1/\varepsilon))$ weights. The formal result is stated in Lemma~\ref{lem:approx-mult} and is due to~\cite{yarotsky2017}.

We close this section with two remarks on the limitations of this setup.

\begin{remark}
We focus on ReLU feed-forward neural networks because they provide a standard and well-understood
baseline in approximation theory. Most of the qualitative arguments of this paper extend to other
activation functions, provided three ingredients remain available:
\begin{itemize}
    \item One should exclude pathological activations for which fixed finite architectures already
    become universal approximators; see \cite{maiorov1999lower}.
    \item The activation should allow the construction of compactly supported bump functions as used in our separation examples. Higher-order ReLU
    activations are natural examples; for a general framework under fairly broad assumptions 
    see \cite{guhring2021encodable}.
    \item For the constructions in Section~\ref{s:Main__ss:realizabilityonly}, one also needs efficient
    approximations to the multiplication operator. 
    Besides ReLU, this is available for essentially all smooth activations, for
    example when the activation has two non-vanishing derivatives on an open interval; see the proof of \cite[Proposition 3.4]{pinkus1999mlp}.
\end{itemize}
\end{remark}

\begin{remark}\label{rem:transformers}
Typically, agentic methods are implemented with transformer architectures rather than feed-forward
ReLU networks. However, ReLU network approximations can be translated into ReLU-transformers;
see \cite{furuya2025transformers} and the quantitative version \cite[Proposition~11]{kratsios2025context}. 
Here we use a slightly different notion of transformers compared to \cite[Proposition~11]{kratsios2025context} and therefore demonstrate an adapted version of the transfer result in Proposition  \ref{prop:transformerification__SparseVersion}.

Hence, the in-context and agentic approximation results established here for ReLU networks also yield corresponding upper bounds for ReLU-based transformer architectures. To extend the separation statements in the same way, one would also need an unattainable target function for the relevant transformer class. 
It follows from \cite[Theorem 8.14]{AnthonyBartlett2009NeuralNetworkLearning} that classifiers based on thresholding the output of transformers with ReLU activation functions have finite VC dimension. By the same argument as in Remark~\ref{rem:unattainableFunctions}, this shows that for a fixed-architecture transformer there exists an unattainable target function.
\end{remark}

\section{Main Results}
\label{s:Main}
\subsection{Monotonicity of adaptivity}
\label{s:Main__ss:Ad}

As noted in the introduction, 
adaptivity does not degrade
a learner's approximation performance 
in either the general setting or under ReLU realizability. 
For completeness, we record this formally below. 

\begin{proposition}[Monotonicity of adaptivity]
\label{prop:monotonicity}
For every task family and every prescribed budget, we have 
\begin{equation} \label{monotone}
    \ICG\leq \AG
    \quad \text{ and } \quad
    \ICR\leq \AR.
\end{equation}
\end{proposition}

\begin{proof}
Let a task family $\mathcal{T}$, a query budget $N\in\mathbb{N}$, and a weight budget $m\in\mathbb{N}$ be fixed.
In the general regimes, recall that only the query budget is imposed.
Let $\Psi_{\rm IC}$ be a general in-context learner with query budget $N$.
We construct a general agentic learner $\Psi_{\rm A}$ that ignores adaptivity by using constant query maps $\{q_n\}_{n=2}^N$. 
Specifically, let $x_1,\dots,x_N$ be the query points of $\Psi_{\rm IC}$.
We define the agentic queries by $x_1(f) \eqdef x_1$ and
\begin{align*}
    x_n(f) = q_n(C_{n-1}(f; x_1(f),\cdots,x_{n-1}(f)) \eqdef q_n(C_{n-1}(f; x_1,\cdots,x_{n-1})) \eqdef x_n, \quad n=2,\dots,N,
\end{align*}
for every $f\in\mathcal{T}$. 
We define the final predictor of $\Psi_{\rm A}$ to be that of $\Psi_{\rm IC}$.
Consequently, 
\begin{equation*}
    \|\Psi_{\rm A}(f) - f\|_{L^{\infty}(\mathcal{X})} =\|\Psi_{\rm IC}(f) - f\|_{L^{\infty}(\mathcal{X})}.
\end{equation*}
Thus, the first relation in \eqref{monotone} follows.
For the realizable comparison, we proceed analogously, mapping an optimal in-context learner to an agentic learner, now reusing the same ReLU network.
\end{proof}

Proposition~\ref{prop:monotonicity} 
reduces the possible relations to the cases listed in the introduction: in each case, there are only two options---strict
improvement or equality.
The following three sections 
realize all 
the remaining cases.

\subsection{Information advantage survives realizability}
\label{s:Main__ss:pathada_realizability}


\begin{figure}[hb!]
    \centering
    \begin{subfigure}[t]{0.48\textwidth}
    \centering
    \includegraphics[width=.95\linewidth]{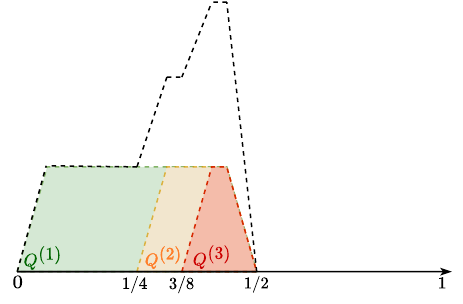}
    \caption{\textbf{The task:} A representative task $f^{\Gamma}$ from \eqref{eq:cubical_path} with $L=3$
    }
    \end{subfigure}%
    \hfill
    \begin{subfigure}[t]{0.48\textwidth}
    \centering
    \includegraphics[width=.75\linewidth]{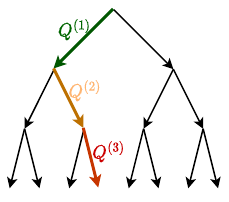}
    \caption{\textbf{Agentic behavior:} The adaptive binary-search protocol implemented by a $\operatorname{ReLU}$-MLP agent.}
    \end{subfigure}
    \caption{\textbf{Cubical-path task family $\mathcal{T}_{d,L}^{\operatorname{path}}$.}
    A representative task from \eqref{eq:cubical_path} is shown in the left figure. Each task $f^\Gamma\in \mathcal{T}_{d,L}^{\operatorname{path}}$ encodes, at every successful query, the location of the next informative query, 
    allowing an adaptive strategy to recover the path sequentially.
    In one dimension, identifying the next relevant sub-cube amounts 
    a \textit{binary search}, shown in the right figure. 
    }
    \label{fig:dyadic}
\end{figure}

We begin with an exemplary task family where adaptive querying is genuinely more informative than every fixed context, and
this advantage survives 
ReLU-realizability.

Fix $d,L\in\mathbb N$ and a parameter $\eta\in(0,1/2)$. For each level
$n\in\{1,\dots,L\}$, let $\mathcal Q_n$ be the dyadic partition of $[0,1]^d$ into cubes of side
length $2^{-n}$. For $Q\in\mathcal Q_n$, write $\ell(Q)=2^{-n}$ and $c(Q)$ for its centre. Define
the central sub-cube
\begin{equation} \label{centralcube}
    Q_\eta\eqdef  \{x\in Q:\ \|x-c(Q)\|_\infty\le \eta\ell(Q)\}
\end{equation}
and the bump function
\[
    \theta_Q(x)
    \eqdef 
    \bigg(1-\frac{\operatorname{dist}_\infty(x,Q_\eta)}{(\tfrac12-\eta)\ell(Q)}\bigg)_+.
\]
A \emph{cubical path of depth $L$} is a nested sequence
\[
\Gamma \eqdef (Q^{(1)},\dots,Q^{(L)}),
\qquad
Q^{(n)}\in\mathcal Q_n,
\qquad
Q^{(L)}\subset\cdots\subset Q^{(1)}.
\]
The associated task, denoted by $f^\Gamma$, is defined for every $x\in [0,1]^d$ to be
\begin{equation}
\label{eq:cubical_path}
f^\Gamma(x)\eqdef  \sum_{n=1}^L \theta_{Q^{(n)}}(x). 
\end{equation}
We denote by $\mathcal{T}_{d,L}^{\mathrm{path}}$ the resulting \emph{cubical-path} family of all such functions $f^{\Gamma}$. 
One example of such a function is shown in Figure \ref{fig:dyadic}.

\begin{theorem}
\label{thm:path}
For every $d,L\in\mathbb N$, the family $\mathcal{T}_{d,L}^{\mathrm{path}}$ satisfies the
following.
\begin{enumerate}
    \item[\rm (i)] There exists a ReLU-realizable agentic learner using $2^dL$ queries that
    reconstructs every task in $\mathcal{T}_{d,L}^{\mathrm{path}}$ exactly.
    \item[\rm (ii)] Every 
    in-context learner with fewer than $2^{d(L-1)}$ queries incurs
    worst-case $L^\infty([0,1]^d)$ approximation error at least $1/2$.
\end{enumerate}
Consequently, for any query budget $N$ satisfying $2^dL \leq N < 2^{d(L-1)}$, we obtain
\[
\ICG<\AG
\quad\text{ and }\quad
\ICR<\AR
\]
on the task family $\mathcal{T}_{d,L}^{\mathrm{path}}$.
\end{theorem}

The idea is simple. An agentic learner recursively identifies the hidden dyadic branch from the cubical path: at each level it
queries the centers of all $2^d$ children of the currently recovered cube and uses the residual
responses to determine the unique child on the path. This procedure 
is ReLU-implementable. 
On the other hand, every non-adaptive sample set with fewer than $2^{d(L-1)}$ points misses an entire cube
at level $L-1$, where two different final children of that cube yield indistinguishable contexts but
different cubical path tasks. 

The proof of Theorem~\ref{thm:path} is given in Appendix~\ref{appx:cubicalpath}.

\subsection{Realizability-only advantage}
\label{s:Main__ss:realizabilityonly}

Our second family shows that adaptivity can yield an advantage after realizability constraints, 
even when general in-context and general agentic methods have the same performance.

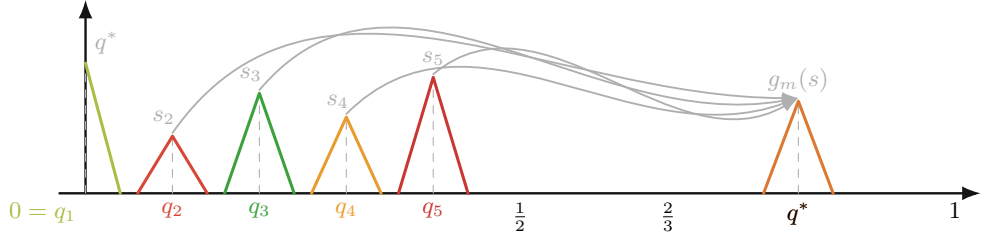
\begin{figure}[ht!]
\centering
\begin{tikzpicture}[>=Latex, line cap=round, line join=round, font=\small, x=11.5cm, y=2.1cm]

\draw[->, line width=0.9pt, draw=nearblack] (-0.03,0) -- (1.03,0) node[below right] {};
\draw[->, line width=0.9pt, draw=nearblack] (0,0) -- (0,1.22);

\def\d{0.04}
\def\qone{0.00}
\def\qtwo{0.10}
\def\qthree{0.20}
\def\qfour{0.30}
\def\qfive{0.40}
\def\qstar{0.82}

\def\hone{0.82}
\def\htwo{0.36}
\def\hthree{0.63}
\def\hfour{0.48}
\def\hfive{0.73}
\def\hstar{0.58}

\draw[dashed, draw=lightgraytext] (\qone,0) -- (\qone,\hone);
\draw[dashed, draw=lightgraytext] (\qtwo,0) -- (\qtwo,\htwo);
\draw[dashed, draw=lightgraytext] (\qthree,0) -- (\qthree,\hthree);
\draw[dashed, draw=lightgraytext] (\qfour,0) -- (\qfour,\hfour);
\draw[dashed, draw=lightgraytext] (\qfive,0) -- (\qfive,\hfive);
\draw[dashed, draw=lightgraytext] (\qstar,0) -- (\qstar,\hstar);

\draw[line width=1.15pt, draw=hatone] (\qone,\hone) -- ({\qone+\d},0);
\draw[line width=1.15pt, draw=hattwo] ({\qtwo-\d},0) -- (\qtwo,\htwo) -- ({\qtwo+\d},0);
\draw[line width=1.15pt, draw=hatthree] ({\qthree-\d},0) -- (\qthree,\hthree) -- ({\qthree+\d},0);
\draw[line width=1.15pt, draw=hatfour] ({\qfour-\d},0) -- (\qfour,\hfour) -- ({\qfour+\d},0);
\draw[line width=1.15pt, draw=hatfive] ({\qfive-\d},0) -- (\qfive,\hfive) -- ({\qfive+\d},0);
\draw[line width=1.15pt, draw=hatsix] ({\qstar-\d},0) -- (\qstar,\hstar) -- ({\qstar+\d},0);

\draw[->, line width=0.7pt, draw=lightgraytext] (\qtwo,\htwo+0.02) to[out=55,in=180] (\qstar,\hstar+0.02);
\draw[->, line width=0.7pt, draw=lightgraytext] (\qthree,\hthree+0.02) to[out=50,in=190] (\qstar,\hstar+0.02);
\draw[->, line width=0.7pt, draw=lightgraytext] (\qfour,\hfour+0.02) to[out=45,in=200] (\qstar,\hstar+0.02);
\draw[->, line width=0.7pt, draw=lightgraytext] (\qfive,\hfive+0.02) to[out=40,in=210] (\qstar,\hstar+0.02);

\node[text=lightgraytext, above right] at (\qone,\hone) {$q^\ast$};
\node[text=lightgraytext, above] at (\qtwo-0.01,\htwo+0.02) {$s_2$};
\node[text=lightgraytext, above] at (\qthree-0.01,\hthree+0.02) {$s_3$};
\node[text=lightgraytext, above] at (\qfour-0.01,\hfour) {$s_4$};
\node[text=lightgraytext, above] at (\qfive,\hfive+0.02) {$s_5$};
\node[text=lightgraytext, above] at (\qstar,\hstar) {$g_m(s)$};

\node[text=hatone, below left] at (\qone,0) {${0=q_1}$};
\node[text=hattwo, below] at (\qtwo,0) {${q_2}$};
\node[text=hatthree, below] at (\qthree,0) {${q_3}$};
\node[text=hatfour, below] at (\qfour,0) {${q_4}$};
\node[text=hatfive, below] at (\qfive,0) {${q_5}$};
\node[text=hatsix, below] at (\qstar,0) {${q^\ast}$};

\node[below] at (0.5,0) {$\frac12$};
\node[below] at (0.82,0) {$q^\ast$};
\node[below] at (0.67,0) {$\frac23$};
\node[below] at (1.00,0) {$1$};

\end{tikzpicture}
\caption{\textbf{Pointed-value family $\mathcal{T}_{N,m}^{\mathrm{val}}$.} A representative task from \eqref{eq:value-family} with $N = 6$ is shown. The fixed hats at $q_1,\dots,q_5$ carry the coefficients $q^\ast,s_2,\dots,s_5$ respectively, and the moving hat centered at $q^\ast\in[2/3,1]$ carries the value $g_m(s)$. 
An unrestricted learner, whether in-context or agentic, can infer the position of the hat at $q^\ast$ by querying the task at the sample point $q_1$.
A general in-context learner can also derive the height of that hat, whereas a ReLU-based learner cannot. Both general and ReLU-based agentic learners can recover the height of the hat at $q^\ast$ by taking one sample.}
\label{fig:value-family}
\end{figure}

Fix $N\ge 3$ and a 
weight budget $m\in\mathbb N$. 
Let $\delta >0$ be sufficiently small so that $\delta<1/(6N)$.
This choice guarantees that, for every $N$-point sample set, there exists a point in $[2/3,1]$ at distance greater than $\delta$ from all samples.
For every $a\in\mathbb{R}$, we define the hat function
\begin{equation} \label{hat}
    h_a(x)\eqdef  \bigg(1-\frac{|x-a|}{\delta}\bigg)_+.
\end{equation}
Let
\begin{equation} \label{theqs}
    q_i\eqdef  \frac{i-1}{2(N-1)},
    \qquad i=1,\dots,N,
\end{equation}
so that $0=q_1<q_2<\cdots<q_N=\frac12$. 

Choose a continuous hard function
$
g_m:[0,1]^{N-2}\to[0,1]
$
such that every ReLU network with at most $m$ weights has uniform approximation error at least
$1/2$ on $g_m$. This function exists according to Remark \ref{rem:unattainableFunctions}.
For $s=(s_2,\dots,s_{N-1})\in[0,1]^{N-2}$ and
$q^\ast\in[2/3,1]$, define the task
\begin{equation}
\label{eq:value-family}
f^{\mathrm{val}}_{s,q^\ast}(x)
\eqdef 
q^\ast h_{q_1}(x)+\sum_{i=2}^{N-1} s_i h_{q_i}(x)+g_m(s)h_{q^\ast}(x).
\end{equation}
By construction, the supports of the hat functions centered at $q_1,\dots,q_N$ are pairwise disjoint. We depict one example in Figure~\ref{fig:value-family}.
Moreover, for each $i$, the support of the hat function at $q_i$ is disjoint from that at $q^\ast\in [2/3,1]$, with separation at least $1/6-2\delta>0$.
Let $\mathcal{T}_{N,m}^{\mathrm{val}}$ denote the resulting \emph{pointed-value} task family of all such $f^{\mathrm{val}}_{s,q^\ast}$.

\begin{theorem}
\label{thm:value}
Let $\eta >0$ and $N \in \mathbb N$ with $N\geq 3$. For sufficiently large $m \in \mathbb N$ depending only on N and $\eta$ the following hold for the family $\mathcal{T}_{N,m}^{\mathrm{val}}$:
\begin{enumerate}
    \item[\rm (i)] Unrestricted in-context learning reconstructs every task exactly from the fixed 
    query points $\{q_1,\dots,q_N\}$.
    \item[\rm (ii)] Unrestricted agentic learning also reconstructs every task exactly with a query budget $N$.
    \item[\rm (iii)] Every ReLU-realizable in-context learner with a sample budget $N$ and a weight budget $m$ 
    incurs worst-case
    $L^\infty([0,1])$ approximation error at least $1/2$. 
    \item[\rm (iv)] There exists a ReLU-realizable agentic learner, with a query budget $N$ and a weight budget $m$, 
    whose worst-case $L^\infty([0,1])$ approximation error on $\mathcal{T}_{N,m}^{\mathrm{val}}$ is at most
    \(\eta\).
\end{enumerate}
Consequently, for \(\eta<1/2\) and $m$ sufficiently large, 
\[
\ICG=\AG
\quad\text{ and }\quad
\ICR<\AR.
\]
\end{theorem}

The proof of Theorem~\ref{thm:value} can be founded in Appendix~\ref{appx:pointedvalue}.

We briefly summarize the main idea of the construction \eqref{eq:value-family}, which underlies the proof.
Here the hard object is the \emph{value}
$g_m(s)$, but it is placed at an accessible 
location $q^\ast$, which can be read directly from the initial contexts (e.g., via $f^{\mathrm{val}}_{s,q^\ast}(q_1)$). 
Unrestricted in-context
and agentic learning are equally powerful because both can compute $g_m(s)$ once $s$ is obtained from the context (e.g., via $f^{\mathrm{val}}_{s,q^\ast}(q_i)$). 
An illustration of this process is given in Figure~\ref{fig:value-family}.
Under ReLU realizability, however, a non-adaptive predictor must approximate the hard map $s\mapsto g_m(s)$ internally, whereas a ReLU-realizable agentic learner can query this value $g_m(s)$
directly and 
reconstruct the 
task $f^{\mathrm{val}}_{s,q^\ast}$ using classical ReLU approximation.

\subsection{Information advantage blocked by realizability}
\label{s:Main__ss:killed-by-realizability}

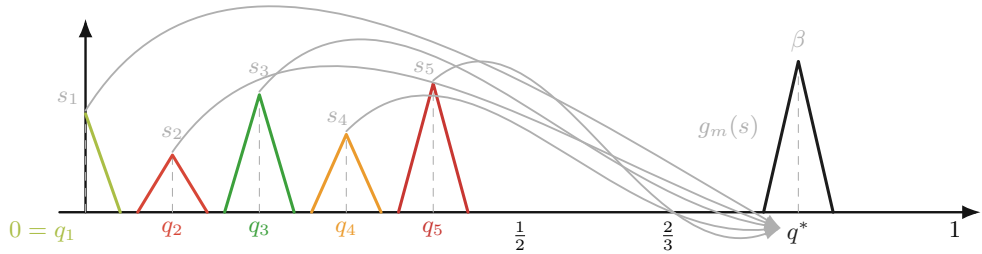
\begin{figure}[htb]
\centering
\begin{tikzpicture}[>=Latex, line cap=round, line join=round, font=\small, x=11.5cm, y=2.1cm]

\draw[->, line width=0.9pt, draw=nearblack] (-0.03,0) -- (1.03,0) node[below right] {};
\draw[->, line width=0.9pt, draw=nearblack] (0,0) -- (0,1.22);

\def\d{0.04}
\def\qone{0.00}
\def\qtwo{0.10}
\def\qthree{0.20}
\def\qfour{0.30}
\def\qfive{0.40}
\def\qsix{0.50}
\def\qstar{0.82}

\def\hone{0.62}
\def\htwo{0.36}
\def\hthree{0.74}
\def\hfour{0.49}
\def\hfive{0.81}
\def\hsix{0.58}
\def\hstar{0.95}

\draw[dashed, draw=lightgraytext] (\qone,0) -- (\qone,\hone);
\draw[dashed, draw=lightgraytext] (\qtwo,0) -- (\qtwo,\htwo);
\draw[dashed, draw=lightgraytext] (\qthree,0) -- (\qthree,\hthree);
\draw[dashed, draw=lightgraytext] (\qfour,0) -- (\qfour,\hfour);
\draw[dashed, draw=lightgraytext] (\qfive,0) -- (\qfive,\hfive);
\draw[dashed, draw=lightgraytext] (\qstar,0) -- (\qstar,\hstar);

\draw[line width=1.15pt, draw=hatone] (\qone,\hone) -- ({\qone+\d},0);
\draw[line width=1.15pt, draw=hattwo] ({\qtwo-\d},0) -- (\qtwo,\htwo) -- ({\qtwo+\d},0);
\draw[line width=1.15pt, draw=hatthree] ({\qthree-\d},0) -- (\qthree,\hthree) -- ({\qthree+\d},0);
\draw[line width=1.15pt, draw=hatfour] ({\qfour-\d},0) -- (\qfour,\hfour) -- ({\qfour+\d},0);
\draw[line width=1.15pt, draw=hatfive] ({\qfive-\d},0) -- (\qfive,\hfive) -- ({\qfive+\d},0);
\draw[line width=1.15pt, draw=charcoalblack] ({\qstar-\d},0) -- (\qstar,\hstar) -- ({\qstar+\d},0);

\draw[->, line width=0.7pt, draw=lightgraytext] (\qone,\hone+0.02) to[out=58,in=150] (\qstar-0.02,-0.1);
\draw[->, line width=0.7pt, draw=lightgraytext] (\qtwo,\htwo+0.02) to[out=55,in=160] (\qstar-0.02,-0.1);
\draw[->, line width=0.7pt, draw=lightgraytext] (\qthree,\hthree+0.02) to[out=50,in=170] (\qstar-0.02,-0.1);
\draw[->, line width=0.7pt, draw=lightgraytext] (\qfour,\hfour+0.02) to[out=45,in=185] (\qstar-0.02,-0.1);
\draw[->, line width=0.7pt, draw=lightgraytext] (\qfive,\hfive+0.02) to[out=40,in=200] (\qstar-0.02,-0.1);

\node[text=lightgraytext, above] at (\qone-0.02,\hone) {$s_1$};
\node[text=lightgraytext, above] at (\qtwo,\htwo+ 0.04) {$s_2$};
\node[text=lightgraytext, above] at (\qthree,\hthree+0.06) {$s_3$};
\node[text=lightgraytext, above] at (\qfour-0.01,\hfour) {$s_4$};
\node[text=lightgraytext, above] at (\qfive-0.01,\hfive) {$s_5$};
\node[text=lightgraytext, above] at (\qstar,\hstar) {$\beta$};
\node[text=lightgraytext, above] at (\qstar-0.08,0.4) {$g_m(s)$};

\node[text=hatone, below left] at (\qone,0) {${0=q_1}$};
\node[text=hattwo, below] at (\qtwo,0) {${q_2}$};
\node[text=hatthree, below] at (\qthree,0) {${q_3}$};
\node[text=hatfour, below] at (\qfour,0) {${q_4}$};
\node[text=hatfive, below] at (\qfive,0) {${q_5}$};
\node[text=charcoalblack, below] at (\qstar,0) {${q^\ast}$};

\node[below] at (0.5,0) {$\frac12$};
\node[below] at (0.67,0) {$\frac23$};
\node[below] at (1.00,0) {$1$};

\end{tikzpicture}
\caption{\textbf{Address-spike family $\mathcal{T}_{N,m}^{\mathrm{addr}}$.} A representative task from \eqref{eq:addrtasks} for $N = 6$ is shown: the coefficients at $q_1,\dots,q_5$ are respectively $s_1,\dots,s_5$, and these values determine the moving location $q^\ast(s)\in[2/3,1]$. The moving hat centered at $q^\ast(s)$ carries a hidden bit $\beta$. We show that no in-context learner can reliably identify the value of $\beta$, and the same is true for a ReLU-based agentic learner. A general agent, however, can compute $q^\ast$ from the samples at $q_1,\dots,q_5$ and then take one additional sample that reveals $\beta$.}
\label{fig:address_spike}
\end{figure}


Our third family of tasks is complementary to the second considered in the previous section. 
This family shows that adaptivity can disappear under realizability constraints.

Fix again $N\ge 3$ and $m\in\mathbb N$, with the same points $q_1,\dots,q_N$ in \eqref{theqs}. 
Let
$
\tau(t)\eqdef 2\min\{t,1-t\},
$ for each $t\in[0,1]$, 
and write for $s=(s_1,\dots,s_{N-1})\in[0,1]^{N-1}$
\begin{equation} \label{hats}
    \widehat s\eqdef  (\tau(s_1),\dots,\tau(s_{N-1})).
\end{equation}
Choose a continuous hard-to-approximate surjective map
$
g_m:[0,1]^{N-1}\to\left[\frac23,1\right]
$
such that every ReLU network with at most $m$ weights has uniform approximation error at least
$1/6$ on $g_m$. Such a function exists by Remark~\ref{rem:unattainableFunctions}. Moreover, since any such hard function is non-constant, one may apply an affine rescaling of its image to ensure surjectivity onto $[2/3,1]$ without decreasing the approximation hardness. 
For $s=(s_1,\dots,s_{N-1})\in[0,1]^{N-1}$ and $\beta\in\{0,1\}$, define
$
q^\ast(s)\eqdef g_m(\widehat s)
$
and
\begin{equation} \label{eq:addrtasks}
    f^{\mathrm{addr}}_{s,\beta}(x)
    \eqdef 
    \sum_{i=1}^{N-1} s_i h_{q_i}(x)
    +
    \beta h_{q^\ast(s)}(x),
\end{equation}
where the hat functions $h_a$ are given in \eqref{hat} with $\delta<1/(6N)$.
Let \(\mathcal T_{N,m}^{\mathrm{addr}}\) denote the resulting \emph{address-spike} family of all such $f^{\mathrm{addr}}_{s,\beta}$. We depict one example of a function in \(\mathcal T_{N,m}^{\mathrm{addr}}\) in Figure \ref{fig:address_spike}.

\begin{theorem}
\label{thm:address}
Let $m,N\in\mathbb{N}$ with $N\geq 3$.
For the family $\mathcal{T}_{N,m}^{\mathrm{addr}}$, the following hold.
\begin{enumerate}
    \item[\rm (i)] 
    There exists a general agentic learner with a query budget $N$ consisting of $N-1$ fixed queries followed by one adaptive query, that reconstructs every task exactly.
    \item[\rm (ii)] Every in-context learner with the sample budget $N$ incurs worst-case
    $L^\infty([0,1])$ approximation error at least $1/2$.
    \item[(iii)] Every ReLU-realizable agentic learner with query budget \(N\) and query maps implemented by ReLU networks with at most $m$ weights incurs worst-case 
    \(L^\infty([0,1])\) approximation error at least \(1/2\). 
\end{enumerate}
Consequently, 
\begin{equation*}
    \ICG<\AG \quad\text{ and }\quad \ICR=\AR.
\end{equation*}
\end{theorem}

The proof of Theorem~\ref{thm:address} is given in Appendix~\ref{appx:address}

We now briefly summarize the main idea of the construction \eqref{eq:addrtasks}, which underlies the proof.
The mechanism here dual to that in the previous section. 
The hard object here is the \emph{location} $q^\ast(s)$ itself. Unrestricted adaptivity helps,
because after reading the static coefficients the learner 
can compute $q^\ast(s)$ and then query $f^{\mathrm{addr}}_{s,\beta}$ there to learn the hidden bit $\beta$. 
An illustration of this process is given Figure~\ref{fig:address_spike}.
However, under ReLU realizability, this advantage changes.
Finding the informative query point
already requires approximating the hard map $s\mapsto q^\ast(s)=g_m(\widehat s)$. Therefore, query maps 
face the
same obstruction as 
an in-context learner.

\section*{Acknowledgements}

A.M.N. and P.C.P. were 
supported by the Austrian Science Fund (FWF) Project P-37010.

\newpage
\bibliographystyle{acm}
\bibliography{Bookkeeping/bibibib}

@book{AnthonyBartlett1999,
  author    = {Martin Anthony and Peter L. Bartlett},
  title     = {Neural Network Learning: Theoretical Foundations},
  publisher = {Cambridge University Press},
  address   = {Cambridge},
  year      = {1999}
}

@inproceedings{chen2022transformers,
  title={Transformers as meta-learners for implicit neural representations},
  author={Chen, Yinbo and Wang, Xiaolong},
  booktitle={European Conference on Computer Vision},
  pages={170--187},
  year={2022},
  organization={Springer}
}

@inproceedings{vonoswald2023what,
  title={What Learning Algorithm is In-Context Learning? Investigations with Linear Models},
  author={von Oswald, Johannes and Niklasson, Eyvind and Sch{\"a}fer, Lennart and Zhao, Zhitao and Ma, Tao and Sch{\"o}lkopf, Bernhard and Domke, Justin},
  booktitle={International Conference on Learning Representations (ICLR)},
  year={2023}
}

@article{cohen2022optimal,
  title={Optimal stable nonlinear approximation},
  author={Cohen, Albert and DeVore, Ronald and Petrova, Guergana and Wojtaszczyk, Przemyslaw},
  journal={Foundations of Computational Mathematics},
  volume={22},
  number={3},
  pages={607--648},
  year={2022},
  publisher={Springer}
}

@inproceedings{furuya2025transformers,
title={Transformers are Universal In-context Learners},
author={Takashi Furuya and Maarten V. de Hoop and Gabriel Peyr{\'e}},
booktitle={The Thirteenth International Conference on Learning Representations},
year={2025},
url={https://openreview.net/forum?id=6S4WQD1LZR}
}

@article{kratsios2025context,
  title={Is in-context universality enough? mlps are also universal in-context},
  author={Kratsios, Anastasis and Furuya, Takashi},
  journal={arXiv preprint arXiv:2502.03327},
  year={2025}
}

@article{petersen2024mathematical,
  title={Mathematical theory of deep learning},
  author={Petersen, Philipp and Zech, Jakob},
  journal={arXiv preprint arXiv:2407.18384},
  year={2024}
}

@article{castro2009active,
  title={Active sensing and learning},
  author={Castro, Rui and Nowak, Robert},
  journal={Foundations and Applications of Sensor Management},
  pages={177--200},
  year={2009}
}

@article{sung1994active,
  title={Active learning for function approximation},
  author={Sung, Kah and Niyogi, Partha},
  journal={Advances in neural information processing systems},
  volume={7},
  year={1994}
}

@article{zhang2024trained,
  title={Trained transformers learn linear models in-context},
  author={Zhang, Ruiqi and Frei, Spencer and Bartlett, Peter L},
  journal={Journal of Machine Learning Research},
  volume={25},
  number={49},
  pages={1--55},
  year={2024}
}

@article{li2025transformers,
  title={Transformers meet in-context learning: A universal approximation theory},
  author={Li, Gen and Jiao, Yuchen and Huang, Yu and Wei, Yuting and Chen, Yuxin},
  journal={arXiv preprint arXiv:2506.05200},
  year={2025}
}

@article{kolmogorov1959varepsilon,
  title={$\varepsilon$-entropy and $\varepsilon$-capacity of sets in function spaces},
  author={Kolmogorov, Andrei Nikolaevich and Tikhomirov, Vladimir Mikhailovich},
  journal={Uspekhi Matematicheskikh Nauk},
  volume={14},
  number={2},
  pages={3--86},
  year={1959},
  publisher={Russian Academy of Sciences, Steklov Mathematical Institute of Russian~…}
}

@inproceedings{akyurek2023what,
title={What learning algorithm is in-context learning? Investigations with linear models},
author={Ekin Aky{\"u}rek and Dale Schuurmans and Jacob Andreas and Tengyu Ma and Denny Zhou},
booktitle={The Eleventh International Conference on Learning Representations },
year={2023},
url={https://openreview.net/forum?id=0g0X4H8yN4I}
}

@article{bordelon2025theory,
  title={Theory of scaling laws for in-context regression: Depth, width, context and time},
  author={Bordelon, Blake and Letey, Mary I and Pehlevan, Cengiz},
  journal={arXiv preprint arXiv:2510.01098},
  year={2025}
}

@article{wu2023many,
  title={How many pretraining tasks are needed for in-context learning of linear regression?},
  author={Wu, Jingfeng and Zou, Difan and Chen, Zixiang and Braverman, Vladimir and Gu, Quanquan and Bartlett, Peter L},
  journal={arXiv preprint arXiv:2310.08391},
  year={2023}
}

@article{angluin1987,
  title   = {Queries and Concept Learning},
  author  = {Angluin, Dana},
  journal = {Machine Learning},
  volume  = {2},
  number  = {4},
  pages   = {319--342},
  year    = {1987}
}

@article{brown2020,
  title   = {Language Models are Few-Shot Learners},
  author  = {Brown, Tom B. and Mann, Benjamin and Ryder, Nick and Subbiah, Melanie and Kaplan, Jared D. and Dhariwal, Prafulla and Neelakantan, Arvind and Shyam, Pranav and Sastry, Girish and Askell, Amanda and Agarwal, Sandhini and Herbert-Voss, Ariel and Krueger, Gretchen and Henighan, Tom and Child, Rewon and Ramesh, Aditya and Ziegler, Daniel M. and Wu, Jeffrey and Winter, Clemens and Hesse, Christopher and Chen, Mark and Sigler, Eric and Litwin, Mateusz and Gray, Scott and Chess, Benjamin and Clark, Jack and Berner, Christopher and McCandlish, Sam and Radford, Alec and Sutskever, Ilya and Amodei, Dario},
  journal = {Advances in Neural Information Processing Systems},
  volume  = {33},
  pages   = {1877--1901},
  year    = {2020}
}

@article{cohn1994,
  title   = {Improving Generalization with Active Learning},
  author  = {Cohn, David A. and Atlas, Les and Ladner, Richard E.},
  journal = {Machine Learning},
  volume  = {15},
  number  = {2},
  pages   = {201--221},
  year    = {1994}
}

@article{dasgupta2011,
  title   = {Two Faces of Active Learning},
  author  = {Dasgupta, Sanjoy},
  journal = {Theoretical Computer Science},
  volume  = {412},
  number  = {19},
  pages   = {1767--1781},
  year    = {2011}
}

@article{garg2022,
  title   = {What Can Transformers Learn In-Context? A Case Study of Simple Function Classes},
  author  = {Garg, Shivam and Tsipras, Dimitris and Liang, Percy S. and Valiant, Gregory},
  journal = {Advances in Neural Information Processing Systems},
  volume  = {35},
  pages   = {30583--30598},
  year    = {2022}
}

@article{lewis2020rag,
  title   = {Retrieval-Augmented Generation for Knowledge-Intensive {NLP} Tasks},
  author  = {Lewis, Patrick and Perez, Ethan and Piktus, Aleksandra and Petroni, Fabio and Karpukhin, Vladimir and Goyal, Naman and K{\"u}ttler, Heinrich and Lewis, Mike and Yih, Wen-tau and Rockt{\"a}schel, Tim and Riedel, Sebastian and Kiela, Douwe},
  journal = {Advances in Neural Information Processing Systems},
  volume  = {33},
  year    = {2020}
}

@inproceedings{park2023generative,
  title     = {Generative Agents: Interactive Simulacra of Human Behavior},
  author    = {Park, Joon Sung and O'Brien, Joseph C. and Cai, Carrie J. and Morris, Meredith Ringel and Liang, Percy and Bernstein, Michael S.},
  booktitle = {Proceedings of the 36th Annual ACM Symposium on User Interface Software and Technology},
  year      = {2023}
}

@inproceedings{yao2023react,
  title     = {{ReAct}: Synergizing Reasoning and Acting in Language Models},
  author    = {Yao, Shunyu and Zhao, Jeffrey and Yu, Dian and Du, Nan and Shafran, Izhak and Narasimhan, Karthik R. and Cao, Yuan},
  booktitle = {The Eleventh International Conference on Learning Representations},
  year      = {2023}
}

@article{yarotsky2017,
  title   = {Error Bounds for Approximations with Deep {ReLU} Networks},
  author  = {Yarotsky, Dmitry},
  journal = {Neural Networks},
  volume  = {94},
  pages   = {103--114},
  year    = {2017}
}

@book{traub1988ibc,
  title     = {Information-Based Complexity},
  author    = {Traub, Joseph F. and Wasilkowski, Grzegorz W. and Wo{\'z}niakowski, Henryk},
  publisher = {Academic Press},
  address   = {New York},
  year      = {1988}
}

@article{packel1987ibc,
  title   = {Recent Developments in Information-Based Complexity},
  author  = {Packel, Edward W. and Wo{\'z}niakowski, Henryk},
  journal = {Bulletin of the American Mathematical Society},
  volume  = {17},
  number  = {1},
  pages   = {9--36},
  year    = {1987}
}

@book{novak2008tractability,
  title     = {Tractability of Multivariate Problems. {V}olume {I}: {L}inear Information},
  author    = {Novak, Erich and Wo{\'z}niakowski, Henryk},
  publisher = {European Mathematical Society},
  address   = {Z\"urich},
  year      = {2008}
}

@article{maiorov1999lower,
  author  = {Maiorov, V. and Pinkus, A.},
  title   = {Lower Bounds for Approximation by {MLP} Neural Networks},
  journal = {Neurocomputing},
  volume  = {25},
  number  = {1--3},
  pages   = {81--91},
  year    = {1999},
  doi     = {10.1016/S0925-2312(98)00111-8}
}

@article{guhring2021encodable,
  author  = {G{\"u}hring, Ingo and Raslan, Mones},
  title   = {Approximation Rates for Neural Networks with Encodable Weights in Smoothness Spaces},
  journal = {Neural Networks},
  volume  = {134},
  pages   = {107--130},
  year    = {2021},
  doi     = {10.1016/j.neunet.2020.11.010}
}

@article{pinkus1999mlp,
  author  = {Pinkus, Allan},
  title   = {Approximation Theory of the {MLP} Model in Neural Networks},
  journal = {Acta Numerica},
  volume  = {8},
  pages   = {143--195},
  year    = {1999},
  doi     = {10.1017/S0962492900002919}
}

@article{yang2025provable,
  title={Provable Low-Frequency Bias of In-Context Learning of Representations},
  author={Yang, Yongyi and Tanaka, Hidenori and Hu, Wei},
  journal={arXiv preprint arXiv:2507.13540},
  year={2025}
}

@article{Caruana1997Multitask,
  author  = {Caruana, Rich},
  title   = {Multitask Learning},
  journal = {Machine Learning},
  volume  = {28},
  number  = {1},
  pages   = {41--75},
  year    = {1997},
  doi     = {10.1023/A:1007379606734}
}

@article{ArgyriouEvgeniouPontil2008ConvexMTL,
  author  = {Argyriou, Andreas and Evgeniou, Theodoros and Pontil, Massimiliano},
  title   = {Convex Multi-Task Feature Learning},
  journal = {Machine Learning},
  volume  = {73},
  number  = {3},
  pages   = {243--272},
  year    = {2008},
  doi     = {10.1007/s10994-007-5040-8}
}

@inproceedings{MishraEtAl2022NaturalInstructions,
  author    = {Mishra, Swaroop and Khashabi, Daniel and Baral, Chitta and Hajishirzi, Hannaneh},
  title     = {Cross-Task Generalization via Natural Language Crowdsourcing Instructions},
  booktitle = {Proceedings of the 60th Annual Meeting of the Association for Computational Linguistics},
  pages     = {3470--3487},
  year      = {2022},
  doi       = {10.18653/v1/2022.acl-long.244}
}

@inproceedings{WangEtAl2022SuperNaturalInstructions,
  author    = {Wang, Yizhong and Mishra, Swaroop and Alipoormolabashi, Pegah and Kordi, Yeganeh and Mirzaei, Amirreza and others},
  title     = {Super-NaturalInstructions: Generalization via Declarative Instructions on 1600+ NLP Tasks},
  booktitle = {Proceedings of the 2022 Conference on Empirical Methods in Natural Language Processing},
  pages     = {5085--5109},
  year      = {2022},
  doi       = {10.18653/v1/2022.emnlp-main.340}
}

@article{lanthaler2026parametric,
  title={The parametric complexity of operator learning},
  author={Lanthaler, Samuel and Stuart, Andrew M},
  journal={IMA Journal of Numerical Analysis},
  volume={46},
  number={2},
  pages={647--712},
  year={2026},
  publisher={Oxford University Press}
}

@inproceedings{LiEtAl2021FNO,
  author    = {Li, Zongyi and Kovachki, Nikola and Azizzadenesheli, Kamyar and Liu, Burigede and Bhattacharya, Kaushik and Stuart, Andrew and Anandkumar, Anima},
  title     = {{F}ourier Neural Operator for Parametric Partial Differential Equations},
  booktitle = {International Conference on Learning Representations},
  year      = {2021}
}

@article{KovachkiEtAl2023NeuralOperator,
  author  = {Kovachki, Nikola and Li, Zongyi and Liu, Burigede and Azizzadenesheli, Kamyar and Bhattacharya, Kaushik and Stuart, Andrew and Anandkumar, Anima},
  title   = {Neural Operator: Learning Maps Between Function Spaces with Applications to {PDEs}},
  journal = {Journal of Machine Learning Research},
  volume  = {24},
  number  = {89},
  pages   = {1--97},
  year    = {2023}
}

@article{KovachkiLanthalerMishra2021FNOTheory,
  author  = {Kovachki, Nikola and Lanthaler, Samuel and Mishra, Siddhartha},
  title   = {On Universal Approximation and Error Bounds for {F}ourier Neural Operators},
  journal = {Journal of Machine Learning Research},
  volume  = {22},
  number  = {290},
  pages   = {1--76},
  year    = {2021}
}

@article{LanthalerMishraKarniadakis2022DeepONetError,
  author  = {Lanthaler, Samuel and Mishra, Siddhartha and Karniadakis, George Em},
  title   = {Error Estimates for {DeepONets}: A Deep Learning Framework in Infinite Dimensions},
  journal = {Transactions of Mathematics and Its Applications},
  volume  = {6},
  number  = {1},
  pages   = {tnac001},
  year    = {2022},
  doi     = {10.1093/imatrm/tnac001}
}

@article{neuman2025stable,
  title={Stable learning using spiking neural networks equipped with affine encoders and decoders},
  author={Neuman, A Martina and Dold, Dominik and Petersen, Philipp Christian},
  journal={Journal of Machine Learning Research},
  volume={26},
  number={246},
  pages={1--49},
  year={2025}
}

@article{kratsios2025generative,
  title={Generative neural operators of log-complexity can simultaneously solve infinitely many convex programs},
  author={Kratsios, Anastasis and Neufeld, Ariel and Schmocker, Philipp},
  journal={arXiv preprint arXiv:2508.14995},
  year={2025}
}

@article{furuya2025one,
  title={One model to solve them all: 2BSDE families via neural operators},
  author={Furuya, Takashi and Kratsios, Anastasis and Possama{\"\i}, Dylan and Raoni{\'c}, Bogdan},
  journal={arXiv preprint arXiv:2511.01125},
  year={2025}
}

@article{petersen2020equivalence,
  title={Equivalence of approximation by convolutional neural networks and fully-connected networks},
  author={Petersen, Philipp and Voigtlaender, Felix},
  journal={Proceedings of the American Mathematical Society},
  volume={148},
  number={4},
  pages={1567--1581},
  year={2020}
}

@article{Barron1993UniversalApproximation,
  author  = {Barron, Andrew R.},
  title   = {Universal Approximation Bounds for Superpositions of a Sigmoidal Function},
  journal = {IEEE Transactions on Information Theory},
  volume  = {39},
  number  = {3},
  pages   = {930--945},
  year    = {1993},
  doi     = {10.1109/18.256500}
}

@inproceedings{mhaskar2017and,
  title={When and why are deep networks better than shallow ones?},
  author={Mhaskar, Hrushikesh and Liao, Qianli and Poggio, Tomaso},
  booktitle={Proceedings of the AAAI conference on artificial intelligence},
  volume={31},
  number={1},
  year={2017}
}

@misc{hu2026transformerapproximationsrelus,
      title={Transformer Approximations from ReLUs}, 
      author={Jerry Yao-Chieh Hu and Mingcheng Lu and Yi-Chen Lee and Han Liu},
      year={2026},
      eprint={2604.24878},
      archivePrefix={arXiv},
      primaryClass={cs.LG},
      url={https://arxiv.org/abs/2604.24878}, 
}

@article{krieg2021function,
  title={Function values are enough for L 2-approximation},
  author={Krieg, David and Ullrich, Mario},
  journal={Foundations of Computational Mathematics},
  volume={21},
  number={4},
  pages={1141--1151},
  year={2021},
  publisher={Springer}
}

@book{pinkus2012n,
  author    = {Pinkus, Allan},
  title     = {n-Widths in Approximation Theory},
  series    = {Ergebnisse der Mathematik und ihrer Grenzgebiete. 3. Folge / A Series of Modern Surveys in Mathematics},
  volume    = {7},
  publisher = {Springer Berlin, Heidelberg},
  year      = {1985},
  doi       = {10.1007/978-3-642-69894-1},
  isbn      = {978-3-642-69894-1}
  }

@article{adcock2025optimal,
  title={Optimal sampling for least-squares approximation},
  author={Adcock, Ben},
  journal={Foundations of Computational Mathematics},
  pages={1--60},
  year={2025},
  publisher={Springer}
}

@article{petersen2018optimal,
  title={Optimal approximation of piecewise smooth functions using deep ReLU neural networks},
  author={Petersen, Philipp and Voigtlaender, Felix},
  journal={Neural Networks},
  volume={108},
  pages={296--330},
  year={2018},
  publisher={Elsevier}
}

@article{kratsios2022universal,
  title={Universal approximation theorems for differentiable geometric deep learning},
  author={Kratsios, Anastasis and Papon, L{\'e}onie},
  journal={Journal of Machine Learning Research},
  volume={23},
  number={196},
  pages={1--73},
  year={2022}
}

@article{kratsios2025beyond,
  title={Beyond Universal Approximation Theorems: Algorithmic Uniform Approximation by Neural Networks Trained with Noisy Data},
  author={Kratsios, Anastasis and Cheng, Tin Sum and Roy, Daniel},
  journal={arXiv preprint arXiv:2509.00924},
  year={2025}
}

@book{AnthonyBartlett2009NeuralNetworkLearning,
  title     = {Neural Network Learning: Theoretical Foundations},
  author    = {Anthony, Martin and Bartlett, Peter L.},
  year      = {2009},
  publisher = {Cambridge University Press},
  address   = {Cambridge}
}

@article{dold2026equivalence,
  title={Equivalence of approximation by networks of single-and multi-spike neurons},
  author={Dold, Dominik and Petersen, Philipp Christian},
  journal={arXiv preprint arXiv:2603.13478},
  year={2026}
}

\appendix

\section{Deep learning models}
\label{s:models}

For completeness, we briefly recall the relevant deep learning models.

\begin{definition}[Fully connected MLPs]
Let $d,D,L\in\mathbb{N}$, and let $d_0,d_1,\dots,d_L\in\mathbb{N}$ such that $d_0=d$ and $d_L=D$.
Let $\sigma\in \mathcal{C}(\mathbb{R})$ be an activation function.
A fully-connected MLP with depth $L$ and activation function $\sigma$ is a map $f:\mathbb{R}^d\to\mathbb{R}^D$ with iterative representation:
\begin{align*} 
    &X^{(0)} 
    \eqdef  X, \\
    &X^{(j+1)} 
    \eqdef  
    \sigma\bullet({\bf A}^{(j)} X^{(j)} + b^{(j)}) \qquad j=0,1,\dots,L-2,
    \qquad f(X)
    \eqdef  
    {\bf A}^{(L-1)} X^{(L-1)} + b^{(L-1)},
\end{align*}
where $\bullet$ denotes component-wise application, ${\bf A}^{(j)}\in\mathbb{R}^{d_{j+1}\times d_j}$, and $b^{(j)}\in\mathbb{R}^{d_{j+1}}$.
\end{definition}

\begin{definition}[Transformers with multi-head attention]
Let $\lambda>0$, and let $d_{\rm in},d_{{\rm key}},d_{\rm out},N\in \mathbb N$. 
Let
${\bf Q},{\bf K}\in \mathbb{R}^{d_{{\rm key}}\times d_{\rm in}}$
and
${\bf V}\in \mathbb{R}^{d_{\rm out}\times d_{\rm in}}$. 
We define the associated attention mechanism
$\operatorname{Attn}(\cdot|{\bf Q},{\bf K},{\bf V},\lambda):\mathbb{R}^{N\times d_{\rm in}}\to \mathbb{R}^{N\times d_{\rm out}}
$
by
\begin{equation} \label{eq:def_attention}
    [\operatorname{Attn}(\mathbf{X}|{\bf Q},{\bf K},{\bf V},\lambda)]_n
    \eqdef 
    \sum_{m=1}^N
    \frac{e^{\lambda \langle {\bf Q}[\mathbf{X}]_n, {\bf K}[\mathbf{X}]_m\rangle/\sqrt{d_{{\rm key}}}}}{
    \sum_{l=1}^N e^{\lambda \langle {\bf Q}[\mathbf{X}]_n, {\bf K}[\mathbf{X}]_l \rangle/\sqrt{d_{{\rm key}}}}}
    {\bf V}[\mathbf{X}]_m, \qquad n=1,\dots,N.
\end{equation}
Here $[\mathbf{X}]_j \in \mathbb{R}^{d_{\rm in}}$ denotes the $j$th row of $\mathbf{X}$, interpreted as a column vector when multiplied by ${\bf Q},{\bf K},{\bf V}$.
A transformer network operates by iterating layers, each consisting of a row-wise activation–bias map followed by an attention mechanism.
Concretely, let $d,D,L\in\mathbb{N}$, and let $d_0,d_1,\dots,d_L, H_0,H_1,\dots,H_L\in\mathbb{N}$ such that $d_0=d$ and $d_L=D$.
For $j=0,1,\dots,L-2$ and $\mathtt{h}=1,\dots,H_j$, let $d^{(j)}_{{\rm key}, \mathtt{h}}, d^{(j)}_{{\rm value}, \mathtt{h}}\in\mathbb{N}$ such that $d_{j+1} = \sum_{{\mathtt{h}}=1}^{H_j} d^{(j)}_{{\rm value}, \mathtt{h}}$.
A $j$th \emph{transformer layer with multi-head attention}, for $j=0,1,\dots,L-2$, is a map
$
\mathcal{T}_j:\mathbb{R}^{N\times d_j}\to \mathbb{R}^{N\times d_{j+1}}
$
that sends an input matrix $\mathbf{X}^{(j)}\in \mathbb{R}^{N\times d_j}$ to
$
\mathcal{T}_j(\mathbf{X}^{(j)})\eqdef  \mathbf{X}^{(j+1)}\in \mathbb{R}^{N\times d_{j+1}},
$
where 
\begin{alignat}{2} 
    \nonumber \mathbf{Z}^{(j)}
    &\eqdef 
    \underbrace{
        \bigoplus_{\mathtt{h}=1}^{H_j}
            \operatorname{Attn}\big(
                \mathbf{X}^{(j)}
                |{\bf Q}_{\mathtt{h}}^{(j)},{\bf K}_{\mathtt{h}}^{(j)},{\bf V}_{\mathtt{h}}^{(j)},\lambda
            \big)
    }_{\text{multi-head attention}} && \\
    \label{eq:representation_transformer} [\mathbf{X}^{(j+1)}]_n
    &\eqdef  
        \sigma\bullet
        \big(
            [\mathbf{Z}^{(j)}]_n + b^{(j)}
        \big),
    &&\qquad n=1,\dots,N.
\end{alignat}
Here $\bigoplus$ denotes concatenation along the second (column) dimension, $
{\bf Q}_{\mathtt{h}}^{(j)},{\bf K}_{\mathtt{h}}^{(j)}\in \mathbb{R}^{d_{{\rm key}, \mathtt{h}}^{(j)}\times d_j}$ and 
$
{\bf V}_{\mathtt{h}}^{(j)}\in \mathbb{R}^{d_{{\rm value}, \mathtt{h}}^{(j)}\times d_j}$, and $b^{(j)}\in\mathbb{R}^{d_{j+1}}$.
Further, the addition in \eqref{eq:representation_transformer} is applied row-wise.
A \emph{transformer with multi-head attention} with depth $L$, obtained by composing the layers $\mathcal{T}_0,\mathcal{T}_1,\dots,\mathcal{T}_{L-2}$, is a map $\mathcal{T}:\mathbb{R}^{N\times d}\to\mathbb{R}^{N\times D}$ sending an input matrix ${\bf X}^{(0)}\eqdef {\bf X}\in\mathbb{R}^{N\times d}$ to $\mathcal{T}({\bf X}) \eqdef {\bf A}^{(L-1)}{\bf X}^{(L-1)} + b^{(L-1)}\in\mathbb{R}^{N\times D}$, for some ${\bf A}^{(L-1)}\in\mathbb{R}^{D\times d_{L-1}}$ and $b^{(L-1)}\in\mathbb{R}^D$.
\end{definition}

Similar to the strategy of~\cite{petersen2020equivalence}, we 
show that every ReLU MLP can be converted into a transformer in a canonical fashion. 
This allows us to deduce a quantitative version of the in-context universality results of~\cite{furuya2025transformers} for the transformer model. 
The next result is the analogue of~\cite[Proposition 11]{kratsios2025context} for our formulation of the transformer used here.


\begin{proposition}[Transformerification of MLPs]
\label{prop:transformerification__SparseVersion}
Let $f:\mathbb{R}^d\to \mathbb{R}^D$ be a MLP with depth $L$. 
Then, for every $\lambda>0$, $f$ can be implemented exactly as a 
transformer with the same depth 
as $f$, exactly one attention head at each layer. 
\end{proposition}

\begin{proof}[{Proof of Proposition~\ref{prop:transformerification__SparseVersion}}]
Let $d_0,d_1,\dots,d_L\in\mathbb{N}$ such that $d_0=d$ and $d_L=D$.
Let $f:\mathbb{R}^d\to\mathbb{R}^D$ admit the iterative representation
\begin{equation}
\label{eq:representation_MLP_}
    \begin{aligned}
    f(X)
    & \eqdef  
    {\bf A}^{(L-1)} X^{(L)}+b^{(L-1)}
    \\
    X^{(j+1)} 
    &
    \eqdef  
    \sigma\bullet\big(
        {\bf A}^{(j)}X^{(j)}
            +
        b^{(j)}
    \big)
    \qquad 
            \mbox{for }  
        j=0,1,\dots,L-2,    
    \\
    X^{(0)} &\eqdef X,
    \end{aligned}
\end{equation}
where ${\bf A}^{(j)}\in\mathbb{R}^{d_{j+1}\times d_j}$ and $b^{(j)}\in\mathbb{R}^{d_{j+1}}$.
Identify each $X^{(j)}\in\mathbb{R}^{d_j}$ with the matrix
$\mathbf{X}^{(j)}\in\mathbb{R}^{1\times d_j}$ whose unique row is $X^{(j)}$.
For each $j=0,1,\dots,L-2$, we define
\[
    {\bf Q}^{(j)}\eqdef  0\in\mathbb{R}^{1\times d_j},
    \qquad
    {\bf K}^{(j)}\eqdef  0\in\mathbb{R}^{1\times d_j},
    \qquad
    {\bf V}^{(j)}\eqdef  {\bf A}^{(j)}\in\mathbb{R}^{d_{j+1}\times d_j}.
\]
It follows from \eqref{eq:def_attention} that for every $\lambda>0$,
$\operatorname{Attn}(\mathbf{X}^{(j)} \mid {\bf Q}^{(j)},{\bf K}^{(j)},{\bf V}^{(j)},\lambda)$ is a matrix with a single row:
\[
    [\operatorname{Attn}(
        \mathbf{X}^{(j)}
        |{\bf Q}^{(j)},{\bf K}^{(j)},{\bf V}^{(j)},\lambda
    )]_1
    =
    {\bf A}^{(j)} {\bf X}^{(j)}.
\]
Therefore, the transformer recursion
\begin{align*} 
    \mathbf{X}^{(j+1)}
    &\eqdef 
    \sigma\bullet \big(
        \operatorname{Attn}(
            \mathbf{X}^{(j)}
            |{\bf Q}^{(j)},{\bf K}^{(j)},{\bf V}^{(j)},\lambda
        )
        +
        b^{(j)}
    \big) 
    \qquad j=0,1,\dots,L-2,
    \\
    \mathbf{X}^{(0)}
    &\eqdef  X,
\end{align*}
satisfies
$
\mathbf{X}^{(j)}=X^{(j)}
$
for every $j=0,1,\dots,L-1$.  
The final affine map can be implemented in the same way, by taking
\[
    {\bf Q}^{(L-1)}\eqdef  0\in\mathbb{R}^{1\times D},
    \qquad
    {\bf K}^{(L-1)}\eqdef  0\in\mathbb{R}^{1\times D},
    \qquad
    {\bf V}^{(L-1)}\eqdef  {\bf A}^{(L-1)}\in\mathbb{R}^{D\times d_{L-1}},
\]
and setting
\[
    \mathcal{T}(X)
    \eqdef 
    \operatorname{Attn}(
        \mathbf{X}^{(L)}
        |{\bf Q}^{(L)},{\bf K}^{(L)},{\bf V}^{(L)},\lambda
    )
    +
    b^{(L)}
    =
    {\bf A}^{(L)}X^{(L)}+b^{(L)}
    =
    f(X).
\]
Thus, we conclude that the designed transformer implements $f$ exactly, with unchanged depth $L$.
\end{proof}

\section{Approximate multiplication by ReLU neural networks}

For convenience, we recall a standard result on approximating multiplication.

\begin{lemma}[{\cite[Proposition 3]{yarotsky2017}}]
\label{lem:approx-mult}
For every \(\varepsilon\in(0,1)\) there exists a ReLU network
\[
\operatorname{Mult}_{\varepsilon}:[0,1]^2\to[0,1]
\]
with \(\mathcal{O}(\log(1/\varepsilon))\) weights such that
$
\sup_{a,b\in[0,1]}
\left|
\operatorname{Mult}_{\varepsilon}(a,b)-ab
\right|
\leq \varepsilon 
$.
\end{lemma}

\section{Proof of Theorem~\ref{thm:path}}
\label{appx:cubicalpath}

\begin{lemma}[Exact ReLU realization of a cubical bump]
\label{lem:cubical-bump}
For every $d,n\in\mathbb N$ and every cube $Q\in\mathcal Q_n$, the bump $\theta_Q$ is exactly
representable by a ReLU network of size and depth $\mathcal O(d)$, with hidden constants depending only
on $\eta$.
\end{lemma}

\begin{proof}
Using the standard exact identities
\[
|u|=\ReLU(u)+\ReLU(-u),
\qquad
\max\{a,b\}=\ReLU(a-b)+b,
\]
one obtains an exact ReLU realization of the coordinatewise maximum
\[
M_d(z_1,\dots,z_d)\eqdef \max\{z_1,\dots,z_d\}
\]
with size and depth $\mathcal O(d)$. 
Moreover, if $Q\in\mathcal Q_n$ has the center $c=(c_1,\dots,c_d)$ and
side length $\ell(Q)=2^{-n}$, then by definition~\eqref{centralcube},
\[
\operatorname{dist}_{\infty}(x,Q_\eta)
=
\max_{i\in[d]}\bigl(|x_i-c_i|-\eta\ell(Q)\bigr)_+.
\]
Hence
\[
\theta_Q(x)
=
\Bigg(
1-
\frac{1}{(\tfrac12-\eta)\ell(Q)}
M_d\Big(\bigl(|x_i-c_i|-\eta\ell(Q)\bigr)_+\Big)_{i=1}^d
\Bigg)_+.
\]
This last expression is built from affine maps, absolute values, maxima, and two extra ReLUs; 
therefore, $\theta_Q$ is exactly ReLU-realizable with the stated network size and depth. 
\end{proof}

We denote by $\Theta(c,\cdot)$ (or $\Theta(c(Q),\cdot)$) the ReLU realization of the bump from Lemma~\ref{lem:cubical-bump}, emphasizing its dependence on the center $c$, in accordance with the translation-invariance of the construction.

\begin{proof}[Proof of Theorem~\ref{thm:path}]
Write $K\eqdef  2^d$ and fix an enumeration $\varepsilon^1,\dots,\varepsilon^K$ of $\{-1,+1\}^d$. 
We construct a ReLU-realizable agentic learner $\Psi$ using $K$ queries per level, hence
$KL=2^dL$ queries in total.
The learner $\Psi$ stores the recovered cube centers level by level. Initialize with the center of
$[0,1]^d$, which is the unique level-$0$ cube. Suppose the center $c^{(n-1)}$ of the path cube at
level $n-1$ has already been recovered. The $K$ dyadic children of that cube have centers
\[
q_j=c^{(n-1)}+2^{-(n+1)}\varepsilon^j,
\qquad j=1,\dots,K.
\]
These query points $q_j$ are affine functions in the current 
transcript $c^{(n-1)}$ and therefore are ReLU-realizable. 
The $K$ queries are performed sequentially: the learner $\Psi$ queries the points $q_j$ one by one, evaluating the task at each 
and recording the response $y_j=f^{\Gamma}(q_j)$. 
Subtract the contribution of the previously recovered ancestor bumps and define the residuals 
\[
\rho_j
\eqdef 
y_j-\sum_{l=1}^{n-1}\Theta_l(c^{(l)},q_j),
\qquad j=1,\dots,K,
\]
where $\Theta_l$ is the ReLU network from Lemma~\ref{lem:cubical-bump} realizing a level $l$ bump
from its center. Exactly one residual equals $1$, namely the one corresponding to the true child on
the hidden path. Indeed, the level $n$ bump function $\theta_{Q^{(n)}}$ equals $1$ at the center of its own cube and $0$ at
the centers of the other children. 
Moreover, since $\eta\in (0,1/2)$, every deeper bump vanishes at the centers of the level $n$ children. Thus
\[
\rho_j=
\begin{cases}
1 &\text{ if }q_j=c(Q^{(n)}),\\
0 &\text{ otherwise.}
\end{cases}
\]
Using
\[
\chi(t)\eqdef  \ReLU(t)-\ReLU(t-1),
\]
which satisfies $\chi(0)=0$ and $\chi(1)=1$, the new center of the path cube at level $n$ is recovered by the ReLU formula
\[
c^{(n)}\eqdef\sum_{j=1}^K \chi(\rho_j)q_j.
\]
Putting these observations together, we conclude that all the query maps are 
ReLU-realizable.
After $L$ levels, the learner $\Psi$ knows $c^{(1)},\dots,c^{(L)}$ and outputs
\begin{equation} \label{finalpred}
    \widehat{F}_{\Psi}(x)\eqdef\sum_{n=1}^L \Theta_n(c^{(n)},x).
\end{equation}
By the construction, $c^{(n)}=c(Q^{(n)})$ for every $n$, so $\widehat{F}_{\Psi}(x)=f^\Gamma(x)$ for all $x\in[0,1]^d$.
It follows from \eqref{finalpred} and Lemma~\ref{lem:cubical-bump} that $\widehat{F}_{\Psi}$ is ReLU-realizable. 
This proves part~(i).

For part~(ii), let us consider 
an in-context learner $\Psi$ using $N<2^{d(L-1)}$ query points $(x_i)_{i=1}^N$.
The dyadic partition $\mathcal Q_{L-1}$ has exactly $2^{d(L-1)}$ cubes, so some cube $Q\in\mathcal Q_{L-1}$ contains no query point. 
Let $Q_1,Q_2$ be two distinct children of $Q$ at level $L$.
Let $\Gamma_1 = (Q^{(1)},\dots,Q^{(L-1)}, Q^{(L)}_1\}$ and $\Gamma_2 = (Q^{(1)},\dots,Q^{(L-1)}, Q^{(L)}_2\}$ be two cubical paths of length $L$ where $Q^{(L-1)}=Q$, and $Q^{(L)}_1=Q_1$, $Q^{(L)}_2=Q_2$.
Recalling definition~\eqref{eq:cubical_path}, the associated tasks $f^{\Gamma_1}, f^{\Gamma_2}\in \mathcal{T}_{d,L}^{\mathrm{path}}$ on $N<2^{d(L-1)}$ query points yield the same context:
\begin{equation*}
    C_N(f^{\Gamma_1};x_1,\dots,x_N) = ((x_i,f^{\Gamma_1}(x_i)))_{i=1}^N
    = ((x_i,f^{\Gamma_2}(x_i)))_{i=1}^N =C_N(f^{\Gamma_2};x_1,\dots,x_N).
\end{equation*}
Since the learner $\Psi$ depends only on this context, it produces the same approximation for both tasks, i.e. $\Psi(f^{\Gamma_1})=\Psi(f^{\Gamma_2})$.
Now let $z=c(Q_1)$. Then $\theta_{Q_1}(z)=1$ and $\theta_{Q_2}(z)=0$; hence 
\[
|f^{\Gamma_1}(z)-f^{\Gamma_2}(z)|=1,
\]
which implies
\begin{equation} \label{thingtocontradict}
    \|f^{\Gamma_1}-f^{\Gamma_2}\|_{L^\infty([0,1]^d)}\ge 1.
\end{equation}
Suppose that
\begin{equation} \label{impossibility}
    \|\Psi(f^{\Gamma_i})-f^{\Gamma_i}\|_{L^\infty([0,1]^d)} < 1/2 \qquad i=1,2.
\end{equation}
Then 
\begin{align*}
    \|f^{\Gamma_1}-f^{\Gamma_2}\|_{L^\infty([0,1]^d)} &\leq \|\Psi(f^{\Gamma_1})-f^{\Gamma_1}\|_{L^\infty([0,1]^d)} + \|\Psi(f^{\Gamma_1})-f^{\Gamma_2}\|_{L^\infty([0,1]^d)} \\
    &= \|\Psi(f^{\Gamma_1})-f^{\Gamma_1}\|_{L^\infty([0,1]^d)} + \|\Psi(f^{\Gamma_2})-f^{\Gamma_2}\|_{L^\infty([0,1]^d)} < 1,
\end{align*}
contradicting \eqref{thingtocontradict}.
Thus at least one of the two tasks incurs an approximation error at least $1/2$ from $\Psi$ in \eqref{impossibility}. 
Since $\Psi$ is an arbitrary in-context learner, it follows that any in-context learner using $N<2^{d(L-1)}$ queries must incur worst-case $L^\infty([0,1]^d)$ approximation error at least $1/2$ on the task family $\mathcal{T}_{d,L}^{\mathrm{path}}$. This proves part~(ii).

As a consequence of part~(i), which provides exact reconstruction for every task in $\mathcal{T}_{d,L}^{\mathrm{path}}$ by a realizable agentic learner, together with Remark~\ref{rem:trivialRelation}, we obtain exact reconstruction for general agentic learners as well. Therefore, for a common sampling budget $N$ with $2^dL\leq N<2^{d(L-1)}$, combining parts~(i) and (ii) yields $\ICG<\AG$. 
For the relation between $\ICR$, $\AR$, we assume that the common weight budget $m$ is sufficiently large to realize the construction in part~(i). Then the same argument implies $\ICR<\AR$. 
\end{proof}

\section{Proof of Theorem~\ref{thm:value}} \label{appx:pointedvalue}

\begin{proof}[Proof of Theorem~\ref{thm:value}]
We prove the four assertions in turn.

\medskip
\noindent
\emph{Proof of (i).}
Let $f^{\mathrm{val}}_{s,q^\ast}\in \mathcal{T}_{N,m}^{\mathrm{val}}$.
Recalling definition~\eqref{eq:value-family}, from the fixed queries $(q_i)_{i=1}^N$, a learner can read off
\[
f^{\mathrm{val}}_{s,q^\ast}(q_1)=q^\ast,
\qquad
f^{\mathrm{val}}_{s,q^\ast}(q_i)=s_i,\quad i=2,\dots,N-1,
\qquad
f^{\mathrm{val}}_{s,q^\ast}(q_N)=0.
\]
Hence an unrestricted learner $\Psi$ recovers both $q^\ast$ and $s$, computes $g_m(s)$, and therefore
reconstructs the task exactly by outputting
\[
\Psi(f^{\mathrm{val}}_{s,q^\ast})(x)
=
q^\ast h_{q_1}(x)+\sum_{i=2}^{N-1} s_i h_{q_i}(x)+g_m(s)h_{q^\ast}(x) = f^{\mathrm{val}}_{s,q^\ast}(x). 
\]

\medskip
\noindent
\emph{Proof of (ii).}
This is immediate from part~(i) and Proposition~\ref{prop:monotonicity}. 

\medskip
\noindent
\emph{Proof of (iii).}
Let $\Psi$ be a ReLU-realizable in-context learner with fixed sample points $\xi_1,\dots,\xi_N\in[0,1]$ and at most $m$ weights. 
Choose $q^\ast\in[2/3,1]$ so that no sample point $\xi_i$ lies in the support of the moving hat $h_{q^\ast}$; this is possible by the choice of $\delta<1/(6N)$ in definition~\eqref{hat}.
For this choice of $q^\ast$, the context received from a task $f^{\mathrm{val}}_{s,q^\ast}$
depends affinely on $s$, by definition~\eqref{eq:value-family}; meaning 
\begin{equation} \label{compressedaffine}
    C_N\bigl(f^{\mathrm{val}}_{s,q^\ast};\xi_1,\dots,\xi_N\bigr)
    =
    A_{\xi}(q^\ast,s)
\end{equation}
for an affine map $A_{\xi}$ where $\xi = (x_1,\dots,x_N)$.
Suppose for contradiction that $\Psi$ achieves worst-case $L^\infty([0,1])$ error strictly
smaller than $1/2$ on $\mathcal{T}_{N,m}^{\mathrm{val}}$. 
Then, by definition \eqref{eq:value-family} and the fact that $q_i\in [0,1/2]$, we have
\[
f^{\mathrm{val}}_{s,q^\ast}(q^\ast)=g_m(s)
\]
for every $s\in[0,1]^{N-2}$. 
Therefore, valuating the approximation $\Psi(f^{\mathrm{val}}_{s,q^\ast})$ at $x=q^\ast$, using definition~\eqref{finpredictor} and \eqref{compressedaffine}, yields
\[
    \bigl|\Psi(f^{\mathrm{val}}_{s,q^\ast})(q^\ast) - f^{\mathrm{val}}_{s,q^\ast}(q^\ast) \bigr|
    =\bigl|\widehat{F}_{\Psi}
    \bigl(A_{\xi}(q^\ast,s),q^\ast\bigr)-g_m(s)\bigr|<\frac12
    \qquad\text{for all }s\in[0,1]^{N-2}.
\]
Because $A_{\xi}(q^\ast,\cdot)$ is affine and $q^\ast$ is fixed, the map
\[
    s\mapsto \widehat{F}_{\Psi}
    \bigl(A_{\xi}(q^\ast,s),q^\ast\bigr) 
\]
is a ReLU network with at most $m$ weights. This contradicts the choice of $g_m$. 
Thus, we conclude that every ReLU-realizable
in-context learner with sample budget $N$ and at most $m$ weights incurs worst-case $L^{\infty}([0,1])$ approximation error at least $1/2$ on $\mathcal{T}_{N,m}^{\mathrm{val}}$.

\medskip
\noindent
\emph{Proof of (iv).}
We construct a realizable agentic learner $\Psi$ as follows.
For a task $f^{\mathrm{val}}_{s,q^\ast}$, the learner $\Psi$ queries $q_1,\dots,q_{N-1}$ sequentially. These query maps are constant and hence ReLU-realizable. This reveals
\begin{equation} \label{priorcontext}
    y_1=f^{\mathrm{val}}_{s,q^\ast}(q_1)=q^\ast,
    \qquad
    y_i=f^{\mathrm{val}}_{s,q^\ast}(q_i)=s_i,\quad i=2,\dots,N-1.
\end{equation}
Using \eqref{priorcontext}, the learner $\Psi$ makes one additional adaptive query at
$
    x_{\mathrm{next}}=y_1=q^\ast
$,
which is an affine transcript-to-query rule \eqref{querymaps} and hence also ReLU-realizable. 
By the construction \eqref{eq:value-family} of $f^{\mathrm{val}}_{s,q^\ast}$, the response is
\[
y_*=f^{\mathrm{val}}_{s,q^\ast}(q^\ast)=g_m(s).
\]

Let \(\varepsilon>0\), to be chosen below, and let
\(\operatorname{Mult}_{\varepsilon}\) be the approximate multiplication network from
Lemma~\ref{lem:approx-mult}. 
Then the final predictor $\widehat{F}_{\Psi}$ of $\Psi$ outputs
\begin{align} \label{approx}
    \nonumber \Psi(f^{\mathrm{val}}_{s,q^\ast})(x)
    &= 
    \widehat{F}_{\Psi}
    \bigl(C_N(f^{\mathrm{val}}_{s,q^\ast};q_1,\dots,q_{N-1},q^{\ast}),x\bigr)\\
    &=
    \operatorname{Mult}_{\varepsilon}(y_1,h_{q_1}(x))
    +
    \sum_{i=2}^{N-1}
    \operatorname{Mult}_{\varepsilon}(y_i,h_{q_i}(x))
    +
    \operatorname{Mult}_{\varepsilon}(y_*,h_{y_1}(x)).
\end{align}
Since all the hats $h_{q_1},\dots,h_{q_{N-1}}, h_{q^\ast}$ are all ReLU hats,
$\widehat{F}_{\Psi}$
is implementable by a ReLU network whose size depends only on $N$ and
\(\varepsilon\).
In particular, for sufficiently large weight budget $m$, the learner $\Psi$ is realizable.
Since all coefficients and all the hat values lie in $[0,1]$, each product in \eqref{approx} is approximated with error
at most \(\varepsilon\). Therefore
\[
    \|\Psi(f^{\mathrm{val}}_{s,q^\ast}) - f^{\mathrm{val}}_{s,q^\ast}\|_{L^\infty([0,1])}
    \leq N\varepsilon
\]
uniformly over $s$ and $q^\ast$. Given \(\eta>0\), we choose \(\varepsilon=\eta/N\). This yields a realizable agentic learner with worst-case approximation error at most $\eta$, as claimed.

Finally, combining Definition~\ref{def:formalDefRelations} with parts~(i), (ii) yields $\ICG=\AG$, while parts~(iii), (iv) yield $\ICR<\AR$.
\end{proof}

\section{Proof of Theorem~\ref{thm:address}} \label{appx:address}

\begin{proof}[Proof of Theorem~\ref{thm:address}]
We prove the three assertions in turn.

\medskip
\noindent
\emph{Proof of (i).}
We construct a general agentic learner $\Psi$ as follows.
For a task $f^{\mathrm{addr}}_{s,\beta}$, the 
learner first queries the fixed points $q_1,\dots,q_{N-1}$. By construction,
this reveals the vector
\[
\bigl(f^{\mathrm{addr}}_{s,\beta}(q_1),\dots,f^{\mathrm{addr}}_{s,\beta}(q_{N-1})\bigr)=s.
\]
The learner then computes $\widehat s$ from \eqref{hats} and queries $g_m$ at $\widehat s$ to obtain
$
    g_m(\widehat s)=q^\ast(s)
$.
Next, the learner makes one additional query of $f^{\mathrm{addr}}_{s,\beta}$ at $q^\ast(s)$. 
Since the supports of the static hats $h_{q_i}$ lie in $[0,1/2+\delta]$
while $q^\ast(s)\in[2/3,1]$, we get from \eqref{eq:addrtasks} a corresponding response
\[
    f^{\mathrm{addr}}_{s,\beta}(q^\ast(s))=\beta.
\]
The learner $\Psi$ now knows both $s$ and $\beta$, and therefore reconstructs the task exactly by
outputting
\[
\Psi(f^{\mathrm{addr}}_{s,\beta})(x)
=\sum_{i=1}^{N-1} s_i h_{q_i}(x)+\beta h_{q^\ast(s)}(x) = f^{\mathrm{addr}}_{s,\beta}(x).
\]

\medskip
\noindent
\emph{Proof of (ii).}
Let $\Psi$ be an in-context learner with fixed sample points $\xi_1,\dots,\xi_N\in[0,1]$.
Because $\delta<1/(6N)$, the union of the intervals $[\xi_i-\delta,\xi_i+\delta]$ has total length
strictly smaller than $1/3$, so there exists a point $y\in[2/3,1]$ with
$|y-\xi_i|>\delta$ for every $i=1,\dots,N$.
Since $\tau$ is surjective onto $[0,2]$, and $g_m$ is surjective onto $[2/3,1]$, we may choose $s\in[0,1]^{N-1}$ such that $y=g_m(\widehat s)=q^\ast(s)$.

Now consider the two tasks $f^{\mathrm{addr}}_{s,0}$ and $f^{\mathrm{addr}}_{s,1}$. Because the moving hat $h_y=h_{q^\ast(s)}$
is supported inside $[y-\delta,y+\delta]$,
both tasks have identical sample values at every
$\xi_i$, yielding the same context 
\begin{equation*}
    C_N(f^{\mathrm{addr}}_{s,0};\xi_1,\dots,\xi_N)
    = C_N(f^{\mathrm{addr}}_{s,1};\xi_1,\dots,\xi_N).
\end{equation*}
As the in-context learner $\Psi$ depends only on this context, it produces the same approximation for both tasks: $\Psi(f^{\mathrm{addr}}_{s,0})=\Psi(f^{\mathrm{addr}}_{s,1})$. 
On the other hand,
\[
f^{\mathrm{addr}}_{s,1}(x)-f^{\mathrm{addr}}_{s,0}(x)=h_{q^\ast(s)}(x),
\]
so at $x=q^\ast(s)$ the two tasks differ by $1$. 
Hence, by the triangle inequality, at least one of the approximations $\|\Psi(f^{\mathrm{addr}}_{s,i})-f^{\mathrm{addr}}_{s,i}\|_{L^{\infty}([0,1])}$ has error
at least $1/2$.
Since $\Psi$ is arbitrary, we conclude that every in-context learner with sample budget $N$ incurs worst-case $L^{\infty}([0,1])$ approximation error at least $1/2$ on $\mathcal T_{N,m}^{\mathrm{addr}}$.

\medskip
\noindent

\emph{Proof of (iii).}
We first note that the 
address map
$
s\mapsto q^\ast(s)=g_m(\widehat s)
$
is 
hard for ReLU networks with at most \(m\) weights. Indeed, suppose that a ReLU network
\(\Phi:[0,1]^{N-1}\to[0,1]\) with at most \(m\) weights satisfies
\[
\|\Phi-g_m\circ \widehat{(\cdot)}\|_{L^\infty([0,1]^{N-1})}<\frac16.
\]
On the subcube \([0,1/2]^{N-1}\), we have \(\tau(t)=2t\), and therefore
\[
    q^\ast(s)=g_m(\widehat{s})=g_m(2s_1,\dots,2s_{N-1}).
\]
Hence the ReLU network
\[
\widetilde \Phi(u_1,\dots,u_{N-1})
\eqdef 
\Phi\left(\frac{u_1}{2},\dots,\frac{u_{N-1}}{2}\right)
\]
has at most the same number of weights and satisfies
\[
\|\widetilde \Phi-g_m\|_{L^\infty([0,1]^{N-1})}<\frac16,
\]
contradicting the choice of \(g_m\).

Now let \(I_i\eqdef  \operatorname{supp}h_{q_i}\), for \(i=1,\dots,N-1\), and
$
    I_*(s)\eqdef  \operatorname{supp}h_{q^\ast(s)}
$.
Assume, for contradiction, that an \(N\)-query ReLU-realizable agentic learner with query maps implementable by ReLU networks with at most \(m\) weights achieves worst-case $L^{\infty}([0,1])$ approximation error strictly smaller than \(1/2\) on $\mathcal T_{N,m}^{\mathrm{addr}}$.

We first claim that the learner must query points in every static support \(I_i\). 
If not, there exists some \(i\) such that no query point falls into \(I_i\). Choose two \(s,s'\in[0,1]^{N-1}\) which agree in all coordinates except the \(i\)th one and satisfy
\[
s_i\neq s_i',
\qquad
\tau(s_i)=\tau(s_i').
\]
For instance, take \(s_i=0\) and \(s_i'=1\). Then $\widehat s=\widehat{s'}$ and therefore
\[
q^\ast(s)=g_m(\widehat s)= g_m(\widehat{s'})=q^\ast(s').
\]
Taking the same value of \(\beta\), the two tasks
\(f^{\mathrm{addr}}_{s,\beta}\) and \(f^{\mathrm{addr}}_{s',\beta}\)
have the same moving spike at $q^\ast(s)=q^\ast(s')$ and the same static spikes at $q_j$, for $j\neq i$.
Thus, they differ \emph{only} on the static support \(I_i\).
On the one hand, since the learner \emph{never} queries in \(I_i\), it receives the same 
context from both tasks and hence produces the same approximation. 
On the other hand, the two tasks \(f^{\mathrm{addr}}_{s,\beta}\) and \(f^{\mathrm{addr}}_{s',\beta}\) differ by \(1\) at \(q_i\), by construction.
It follows that at least one of the corresponding approximation errors incurred by the learner is at least \(1/2\), a contradiction.

We next claim that the learner must query points in the moving support \(I_*(s)\). Otherwise, the two tasks \(f^{\mathrm{addr}}_{s,0}\) and \(f^{\mathrm{addr}}_{s,1}\) generate identical 
contexts to the learner, by the now routine argument. 
The learner would therefore output the same approximation for both tasks. However, since
\[
f^{\mathrm{addr}}_{s,1}(q^\ast(s))
-
f^{\mathrm{addr}}_{s,0}(q^\ast(s))
=
1,
\]
we again conclude that one of the two approximation errors is at least \(1/2\), a contradiction.

Thus, for every \(s\), the \(N\) queries are fully accounted for: \(N-1\) queries go to the static supports and one query goes to the moving support. Since the static supports are fixed, known in advance, pairwise disjoint, and carry independent coefficients, sampling them in arbitrary order reveals exactly the corresponding coordinates \(s_i\).
Therefore, for the lower bound, we may reorder the query points and assume that the learner first reads the static values \(s_1,\dots,s_{N-1}\) and then queries the induced moving support. 
Consequently, the last 
query map is given by a ReLU neural network
\[
\Phi:[0,1]^{N-1}\to[0,1]
\]
with at most \(m\) weights. Since this querying must hit $I_*(s)=\operatorname{supp}h_{q^\ast(s)} = [q^\ast(s)-\delta, q^\ast(s)+\delta]$
for every \(s\), we have
\[
|\Phi(s)-q^\ast(s)|\leq \delta
\quad
\text{ for all }\quad s\in[0,1]^{N-1}.
\]
Therefore
\[
\|\Phi-g_m\circ \widehat{(\cdot)}\|_{L^\infty([0,1]^{N-1})}
\leq \delta.
\]
Since $\delta<1/(6N)\leq 1/18$, this contradicts the hardness of the address map $s\mapsto q^\ast(s)=g_m(\widehat s)$, as established at the beginning of the proof. Hence, every ReLU-realizable agentic learner of this form incurs worst-case \(L^\infty([0,1])\) approximation error at least \(1/2\) on $\mathcal T_{N,m}^{\mathrm{addr}}$.

Finally, parts~(i),~(ii) together yield $\ICG<\AG$, while part~(iii) implies $\ICR = \AR$ since a learner with constant predictor $x\mapsto 1/2$ achieves worst-case error $1/2$.
\end{proof}

\end{document}